\documentclass[12pt, final]{l4dc2022}

\usepackage[utf8]{inputenc} 
\usepackage[T1]{fontenc}    
\usepackage{hyperref}       
\usepackage{url}            
\usepackage{booktabs}       
\usepackage{amsfonts}       
\usepackage{nicefrac}       
\usepackage{microtype}      
\usepackage{enumitem}
\usepackage{graphicx}
\usepackage{algorithmic}
\usepackage{bbm}
\usepackage{amsmath,amssymb}
\usepackage[utf8]{inputenc}
\usepackage[english]{babel}
\usepackage{color}
\usepackage{xcolor}
\usepackage{stackrel}
\usepackage{multirow}
\usepackage{algorithm}
\usepackage{wrapfig}
\usepackage[normalem]{ulem}
\usepackage[most]{tcolorbox}
\usepackage{xfrac}
\usepackage{caption}

\DeclareMathOperator*{\argmin}{arg\,min}
\DeclareMathOperator*{\supp}{supp}
\DeclareMathOperator*{\rsupp}{rsupp}

\newtheorem{assumption}{Assumption}

\definecolor{myblue2}{HTML}{0D6CCB}
\definecolor{mygreen}{HTML}{3BBA08}
\definecolor{myred}{HTML}{C12505}

\title[i-SpaSP]{i-SpaSP: Structured Neural Pruning via Sparse Signal Recovery}
\usepackage{times}

\author{%
 \Name{Cameron R. Wolfe} \Email{crw13@rice.edu}\\
 \addr Department of Computer Science, Rice University, Houston, TX, USA.
 \AND
 \Name{Anastasios Kyrillidis} \Email{anastasios@rice.edu}\\
 \addr Department of Computer Science, Rice University, Houston, TX, USA.
}

\begin{document}

\maketitle

\begin{abstract}%
    We propose a novel, structured pruning algorithm for neural networks---the \textbf{i}terative, \textbf{Spa}rse \textbf{S}tructured \textbf{P}runing algorithm, dubbed as i-SpaSP.
    Inspired by ideas from sparse signal recovery, i-SpaSP operates by iteratively identifying a larger set of important parameter groups (e.g., filters or neurons) within a network that contribute most to the residual between pruned and dense network output, then thresholding these groups based on a smaller, pre-defined pruning ratio.
    For both two-layer and multi-layer network architectures with ReLU activations, we show the error induced by pruning with i-SpaSP decays polynomially, where the degree of this polynomial becomes arbitrarily large based on the sparsity of the dense network's hidden representations.
    In our experiments, i-SpaSP is evaluated across a variety of datasets (i.e., MNIST, ImageNet, and XNLI) and architectures (i.e., feed forward networks, ResNet34, MobileNetV2, and BERT), where it is shown to discover high-performing sub-networks and improve upon the pruning efficiency of provable baseline methodologies by several orders of magnitude.
    Put simply, i-SpaSP is easy to implement with automatic differentiation, achieves strong empirical results, comes with theoretical convergence guarantees, and is efficient, thus distinguishing itself as one of the few computationally efficient, practical, and provable pruning algorithms. 
\end{abstract}

\begin{sourcecode}
\href{https://github.com/wolfecameron/i-SpaSP}{\textcolor{blue}{$\text{https://github.com/wolfecameron/i-SpaSP}$}}
\end{sourcecode}

\section{Introduction}
\noindent \textbf{Background.}
Neural network pruning has garnered significant recent interest \citep{lth, rethink_pruning, conv_filt_prune}, as obtaining high-performing sub-networks from larger, dense networks enables a reduction in the computational and memory overhead of neural network applications \citep{han2015deep, han2015learning}.
Many popular pruning techniques are based upon empirical heuristics that work well in practice \citep{stable_lth, thinet, he_pruning}.
Generally, these methodologies introduce some notion of ``importance'' for network parameters (or groups of parameters) and eliminate parameters with negligible importance.

The empirical success of pruning methodologies inspired the development of pruning algorithms with theoretical guarantees \citep{sipping_nn, liebenwein2019provable, mussay2019data, baykal2018data, whats_hidden}.
Among such work, greedy forward selection (GFS) \citep{provable_subnetworks, log_pruning}---inspired by the Frank-Wolfe algorithm \citep{frank1956algorithm}---differentiated itself as a methodology that performs well in practice and provides theoretical guarantees.
However, GFS is inefficient in comparison to popular pruning heuristics.

\noindent
\textbf{This Work.}
We leverage greedy selection \citep{cosamp, khanna2018iht} to develop a structured pruning algorithm that is provable, practical, and efficient.
This algorithm, called \textbf{i}terative, \textbf{Spa}rse \textbf{S}tructured \textbf{P}runing (i-SpaSP), iteratively estimates the most important parameter groups\footnote{``Parameter groups'' refers to the minimum structure used for pruning (e.g., neurons or filters).} within each network layer, then thresholds this set of parameter groups based on a pre-defined pruning ratio---a similar procedure to the CoSAMP algorithm for sparse signal recovery \citep{cosamp}. 
Theoretically, we show for two and multi-layer networks that $i)$ the output residual between pruned and dense networks decays polynomially with respect to the size of the pruned network and $ii)$ the order of this polynomial increases as the dense network's hidden representations become more sparse.\footnote{To the best of the authors' knowledge, we are the first to provide theoretical analysis showing that the quality of pruning depends upon the sparsity of representations within the dense network.}
In experiments, we show that i-SpaSP is capable of discovering high-performing sub-networks across numerous different models (i.e., two-layer networks, ResNet34, MobileNetV2, and BERT) and datasets (i.e., MNIST, ImageNet, and XNLI). 
i-SpaSP is simple to implement and significantly improves upon the runtime of GFS variants.

\section{Preliminaries} \label{S:prelim}
\noindent
\textbf{Notation.}
Vectors and scalars are denoted with lower-case letters.
Matrices and certain constants are denoted with upper-case letters.
Sets are denoted with upper-case, calligraphic letters (e.g., $\mathcal{G}$) with set complements $\mathcal{G}^c$.
We denote $[n] = \{0, 1, \dots, n\}$.
For $x \in \mathbb{R}^N$, $\|x\|_p$ is the $\ell_p$ vector norm.
$x_s$ is the $s$ largest-valued components of $x$, where $|x| \geq s$.
$\supp(x)$ returns the support of $x$.
For index set $\mathcal{G}$, $x|_{\mathcal{G}}$ is the vector with non-zeros at the indices in $\mathcal{G}$.
For $X \in \mathbb{R}^{m \times n}$, $\|X\|_F$ and $X^\top$ represent the Frobenius norm and transpose of $X$. 
The $i$-th row and $j$-th column of $X$ are given by $X_{i, :}$ and $X_{:, j}$, respectively.
$\rsupp(X)$ returns the row support of $X$ (i.e., indices of non-zero rows).
For index set $\mathcal{G}$, $X_{\mathcal{G}, :}$ and $X_{:, \mathcal{G}}$ represent row and column sub-matrices, respectively, that contain rows or columns with indices in $\mathcal{G}$.
$\mu(X): \mathbb{R}^{m \times n} \longrightarrow \mathbb{R}^m$ sums over columns of a matrix (i.e., $\mu(X) = \sum_{i} X_{:, i}$), while $\texttt{vec}(X): \mathbb{R}^{m \times n} \longrightarrow \mathbb{R}^{mn}$ stacks columns of a matrix.
$X \in \mathbb{R}^{m \times n}$ is $p$-row-compressible with magnitude $R \in \mathbb{R}^{\geq 0}$ if $|\mu(X)|_{(i)} \leq \frac{R}{i^{\frac{1}{p}}},~~\forall i \in [m],$ where $|\cdot|_{(i)}$ denotes the $i$-th sorted vector component (in magnitude).
Lower $p$ values indicate a nearly row-sparse matrix and vice versa.

\medskip
\noindent
\textbf{Network Architecture.}
Our analysis primarily considers two-layer, feed forward networks\footnote{Though \eqref{eq:2layer_nn} has no bias, a bias term could be implicitly added as an extra element within the input and weight matrices.}:
\begin{align}
    f(X, \mathcal{W}) = W^{(1)}\cdot \sigma(W^{(0)} \cdot X) \label{eq:2layer_nn}
\end{align}
The network's input, hidden, and output dimensions are given by $d_{in}$, $d_{hid}$, and $d_{out}$.
$X \in \mathbb{R}^{d_{in} \times B}$ stores the full input dataset with $B$ examples.
$W^{(0)} \in \mathbb{R}^{d_{hid} \times d_{in}}$ and $W^{(1)} \in \mathbb{R}^{d_{out} \times d_{hid}}$ denote the network's weight matrices.
$\sigma(\cdot)$ denotes the ReLU activation function and $H = \sigma(W^{0}\cdot X)$ stores the network hidden representations across the dataset.
We also extend our analysis to multi-layer networks with similar structure; see Appendix \ref{A:multi_layer} for more details.


\section{Related Work} \label{S:related_work}
\noindent \textbf{Pruning.}
Neural network pruning strategies can be roughly separated into structured \citep{eie, conv_filt_prune, net_slimming, provable_subnetworks, log_pruning} and unstructured \citep{rigl, how_lth_wins, lth, deep_comp} variants.
Structured pruning, as considered in this work, prunes parameter groups instead of individual weights, allowing speedups to be achieved without sparse computation \citep{conv_filt_prune}.
Empirical heuristics for structured pruning include removing parameter groups with low $\ell_1$ norm \citep{conv_filt_prune, net_slimming}, measuring the gradient-based sensitivity of parameter groups \citep{sipping_nn, preserve_grad_flow, discrim_prune}, preserving network output \citep{he_pruning, thinet, nisp}, and more \citep{suau2020filter, chin2019legr, huang2018data, molchanov2016pruning}.
Pruning typically follows a three-step process of pre-training, pruning, and fine-tuning \citep{conv_filt_prune, rethink_pruning}, where pre-training is typically the most expensive component \citep{early_bert, early_bird}.

\medskip
\noindent \textbf{Provable Pruning.}
Empirical pruning research inspired the development of theoretical foundations for network pruning, including sensitivity-based analysis \citep{sipping_nn, liebenwein2019provable}, coreset methodologies \citep{mussay2019data, baykal2018data}, random network pruning analysis \citep{pruning_is_all_you_need, log_pruning_is_all_you_need, subset_lth, whats_hidden}, and generalization analysis \citep{zhang2021lottery}.
GFS---analyzed for two-layer \citep{provable_subnetworks, wolfe2021provably} and multi-layer \citep{log_pruning} networks---was one of the first pruning methodologies to provide both strong empirical performance and theoretical guarantees. 

\medskip
\noindent \textbf{Greedy Selection.} 
Greedy selection efficiently discovers approximate solutions to combinatorial optimization problems \citep{frank1956algorithm}.
Many algorithms and frameworks for greedy selection exist; e.g., Frank-Wolfe \citep{frank1956algorithm}, sub-modular optimization \citep{submod_opt}, CoSAMP \citep{cosamp}, and iterative hard thresholding \citep{khanna2018iht}.
Frank-Wolfe has been used within GFS and to train deep neural networks \citep{bach_fw, frankwolfe_nn}, thus forming a connection between greedy selection and deep learning.
Furthering this connection, we leverage CoSAMP \citep{cosamp} to formulate our proposed methodology.

\section{Methodology} \label{S:method}

i-SpaSP is formulated for two-layer networks in Algorithm \ref{A:cosamp_prune}, where the pruned model size $s$ and total iterations $T$ are fixed.
$\mathcal{S}$ stores active neurons in the pruned model, which is refined over iterations. 

\begin{wrapfigure}{L}{0.49\textwidth}
\vspace{-0.3cm}
\begin{minipage}{0.49\textwidth}
\begin{algorithm}[H]
\SetAlgoLined
\textbf{Parameters:}~$T$,~$s$; $\mathcal{S}$ := $\emptyset$; $t$ := 0 
\vspace{0.1cm} \\
\textcolor{myblue2}{\# compute hidden representation} \\
$H = \sigma(W^{(0)} \cdot X)$\\
$h = \mu(H)$\\\vspace{-.8cm}\\

\textcolor{myblue2}{\# compute dense network output} \\
$U = W^{(1)} \cdot H$ \\
$V = U$\\\vspace{-.3cm}\\
\While{t < T}{
$t = t + 1$\\\vspace{-.3cm}\\
\textcolor{myblue2}{\# \textbf{Step I:} Estimating Importance} \\
$Y = (W^{(1)})^\top \cdot V$\\
$y = \mu(Y)$\\
$\Omega = \supp(y_{2s})$ \\\vspace{-.3cm}\\
\textcolor{myblue2}{\# \textbf{Step II:} Merging and Pruning} \\
$\Omega^\star = \Omega \cup \mathcal{S}$ \\
$b = h|_{\Omega^\star}$ \\
$\mathcal{S} = \supp(b_s)$\\\vspace{-.3cm}\\
\textcolor{myblue2}{\# \textbf{Step III:} Computing New Residual} \\
$V = U - W^{(1)}_{:, \mathcal{S}} \cdot H_\mathcal{S, :}$
}
\textcolor{myblue2}{\# return pruned model with neurons in $\mathcal{S}$}\\
\texttt{return} $\{W^{(0)}_{\mathcal{S}, :}, W^{(1)}_{:, \mathcal{S}}\}$
\caption{i-SpaSP for Two-Layer Networks}
\label{A:cosamp_prune}
\end{algorithm}
\end{minipage}
\vspace{-1.4cm}
\end{wrapfigure}

\medskip
\noindent \textbf{Why does this work?}
Each iteration of i-SpaSP follows a three-step procedure in Algorithm \ref{A:cosamp_prune}:
\begin{itemize}[leftmargin=1.6cm]
 \item [\textbf{Step I:}] Compute neuron ``importance'' $Y$ given the current residual matrix $V$. \vspace{-0.2cm}
\end{itemize}

\vspace{-.3cm}
\begin{itemize}[leftmargin=1.65cm]
 \item[\textbf{Step II:}] Identify $s$ neurons within the combined set of important and active neurons (i.e., $\Omega \cup \mathcal{S}$) with the largest-valued hidden representations.
\end{itemize}

\vspace{-.6cm}
\begin{itemize}[leftmargin=1.7cm]
 \item[\textbf{Step III:}] Update $V$ with respect to the new pruned model estimate.
\end{itemize}

\noindent
We now provide intuition regarding the purpose of each individual step within i-SpaSP.

\medskip
\noindent
\textbf{Estimating Importance.}
$Y_{ij}$ is the importance of hidden neuron $i$ with respect to dataset example $j$.
We can characterize the discrepancy between pruned and dense network output $U'$ and $U$ as:
\begin{align}
    \mathcal{L}(U, U') = \frac{1}{2}\|W^{(1)}\cdot H - U' \|_F^2. \label{eq:prune_obj}
\end{align}
Considering $U'$ fixed, $\nabla_H \mathcal{L}(U, U') = (W^{(1)})^\top \cdot V$.
As such, if $Y_{ij}$ is a large, positive (negative) value, decreasing (increasing) $H_{ij}$ will decrease the value of $\mathcal{L}$ locally.
Then, because $H$ is non-negative and cannot be modified via pruning, one can realize that the best methodology of minimizing \eqref{eq:prune_obj} is including neurons with large, positive importance values within $\mathcal{S}$, as in Algorithm \ref{A:cosamp_prune}.

\medskip
\noindent
\textbf{Merging and Pruning.}
The $2s$ most-important neurons---based on $\mu(Y)$ components---are selected and merged with $\mathcal{S}$, allowing a larger set of neurons (i.e., more than $s$) to be explored.
From here, $s$ neurons with the largest-valued components in $\mu(H)$ are sub-selected from this combined set to form the next pruned model estimate.
Because hidden representation values are not considered in importance estimation, performing this two-step merging and pruning process ensures neurons within $\mathcal{S}$ have both large hidden activation and importance values, which together indicate a meaningful impact on network output.

\medskip
\noindent
\textbf{Computing the New Residual.}
The next pruned model estimate is used to re-compute $V$, which can be intuitively seen as updating $U'$ in \eqref{eq:prune_obj}.
As such, importance in Algorithm \ref{A:cosamp_prune} is based on the current pruned and dense network residual and \eqref{eq:prune_obj} is minimized over successive iterations.

\subsection{Implementation} \label{S:implementation}

\medskip
\noindent \textbf{Automatic Differentiation.}
Because $Y = \nabla_H \mathcal{L}(U, U')$, importance within i-SpaSP can be computed efficiently using automatic differentiation \citep{paszke2017automatic, tensorflow}; see Algorithm \ref{A:autodiff} for a PyTorch-style example.
Automatic differentiation simplifies importance estimation and allows it to be run on a GPU, making the implementation efficient and parallelizable.
Because the remainder of the pruning process only leverages basic sorting and set operations, the overall implementation of i-SpaSP is both simple and efficient.

\medskip
\noindent \textbf{Other Architectures.} 
Algorithm \ref{A:autodiff} can be easily generalized to more complex network modules (i.e., beyond feed-forward layers) using automatic differentiation.
Notably, convolutional filters or attention heads can be pruned using the importance estimation from Algorithm \ref{A:autodiff} if $\texttt{sum}(\cdot)$ is performed over both batch and spatial dimensions. 
Furthermore, i-SpaSP can be used to prune multi-layer networks by greedily pruning each layer of the network from beginning to end.

\medskip
\noindent \textbf{Large-Scale Datasets.}
Algorithm \ref{A:cosamp_prune} assumes the entire dataset is stored within $X$.
For large-scale experiments, such an approach is not tractable.
As such, we redefine $X$ within experiements to contain a subset or mini-batch of data from the full dataset, allowing the pruning process to be performing in an approximate---but computationally tractable---manner.
To improve the quality of this approximation, a new mini-batch is sampled during each i-SpaSP iteration, but the size of such mini-batches becomes a hyperparameter of the pruning process.\footnote{Both GFS \citep{provable_subnetworks} and multi-layer GFS \citep{log_pruning} adopt a similar mini-batch approach during pruning.}

\begin{minipage}{0.5\textwidth}
\begin{tcolorbox}
    \begin{footnotesize}
    $H$\texttt{.requires\_grad} $:=$ \texttt{True} \\
    \texttt{with torch.no\_grad():} \\
    \text{~~~~~~~~}\texttt{prune\_out} $:=$ \texttt{prune\_layer}$(H_{\mathcal{S}, :})$ \\
    \texttt{dense\_out} $:=$ \texttt{dense\_layer}$(H)$ \\
    \texttt{obj} $:=$ \texttt{sum} $\left(\frac{1}{2}(\texttt{dense\_out} - \texttt{prune\_out})^2\right)$\\
    \texttt{obj.backward()}\\
    \texttt{importance} $:= \texttt{sum}(H\texttt{.grad}, \texttt{dim}=0)$\\
    \texttt{return~importance}
    \end{footnotesize}
\end{tcolorbox}
\captionof{algorithm}{i-SpaSP importance computation via automatic differentiation.}
\label{A:autodiff}
\end{minipage}\hspace{0.5cm}
\begin{minipage}{0.45\textwidth}
    \centering
    \includegraphics[width=0.95\textwidth]{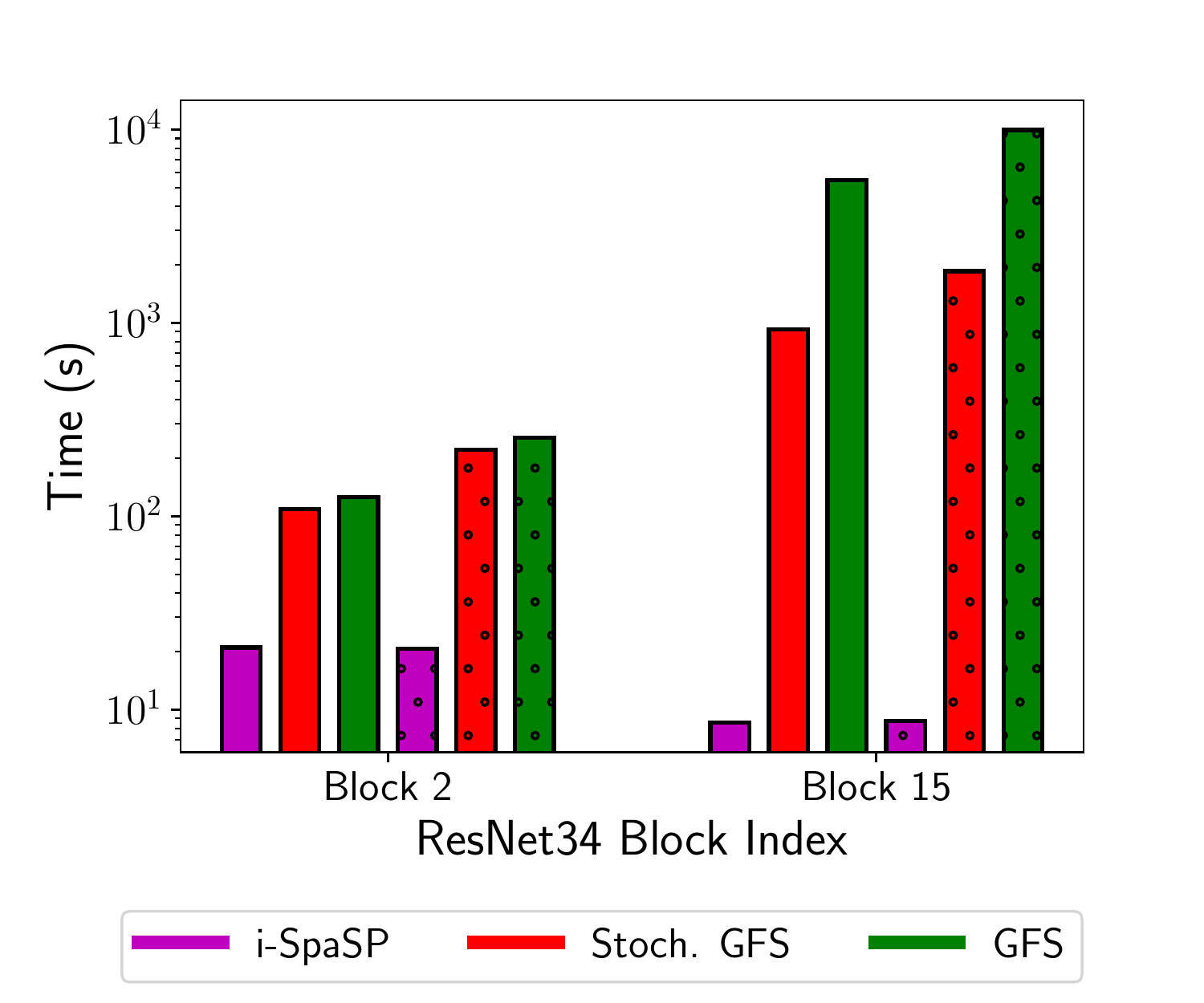}
    \vspace{-0.4cm}
    \captionof{figure}{Runtime of pruning ResNet34 blocks with i-SpaSP and GFS variants to 20\% (i.e., plain) or 40\% (i.e., dotted) of filters.}
    \label{fig:cs_time}
\end{minipage}

\subsection{Computational Complexity and Runtime Comparisons} \label{S:comp_complex}
Denote the complexity of matrix-matrix multiplication as $\xi$.
The complexity of pruning a network layer with $T$ iterations of i-SpaSP is $\mathcal{O}(T\xi + T d_{hid}\log(d_{hid}))$.
GFS has a complexity of $\mathcal{O}(s \xi d_{hid})$, as it adds a single neuron to the (initially empty) network layer at each iteration by exhaustively searching for the neuron that minimizes training loss.
Though later GFS variants achieve complexity of $\mathcal{O}(s \xi)$ \citep{log_pruning}, i-SpaSP is more efficient in practice because the forward pass (i.e. $\mathcal{O}(\xi)$) dominates the pruning procedure and $T \ll s$ (e.g., $T=20$ in Section \ref{S:experiments}).

As a practical runtime comparison, we adopt a ResNet34 model \citep{resnet} and measure wall-clock pruning time\footnote{We prune the 2nd (64 channels) and 15th (512 channels) convolutional blocks. We choose blocks in different network regions to view the impact of channel and spatial dimension on pruning efficiency.} with i-SpaSP and GFS. 
We use the public implementation of GFS \citep{gfs_github} and test both stochastic and vanilla variants\footnote{Vanilla GFS exhaustively searches neurons within each iteration, while the stochastic variant randomly selects 50 neurons to search per iteration.}.
i-SpaSP uses settings from Section \ref{S:experiments}, and selected ResNet blocks are pruned to ratios of 20\% or 40\% of original filters; see Figure \ref{fig:cs_time}.
i-SpaSP significantly improves upon the runtime of GFS variants; e.g., i-SpaSP prunes Block 15 in roughly 10 seconds, while GFS takes over 1000 seconds in the best case. 
Furthermore, unlike GFS, the runtime of i-SpaSP is not sensitive to the size of the pruned network (i.e., wall-clock time is similar for ratios of 20\% and 40\%), though i-SpaSP does prune later network layers faster than earlier layers.

\section{Theoretical Results} \label{S:theory}
Proofs are deferred to Appendix \ref{A:proofs}.
The dense network is pruned from $d_{hid}$ to $s$ neurons via i-SpaSP.
We assume that $W^{(1)}$ satisfies the restricted isometry property (RIP) \citep{rip}:

\begin{assumption}[Restricted Isometry Property (RIP) \citep{rip}] \label{assm:rip}
    Denote the $r$-th restricted isometry constant as $\delta_r$ and assume $\delta_r \leq 0.1$ for $r = 4s$.\footnote{This is a numerical assumption adopted from \citep{cosamp}, which holds for Gaussian matrices of size $\mathbb{R}^{m \times n}$ when $m \geq \mathcal{O}\left(r\log(\frac{n}{r})\right)$.} Then, $W^{(1)}$ satisfies the RIP with constant $\delta_r$ if  $ (1 - \delta_r)\|x\|_2^2 \leq \|W^{(1)} \cdot x\|_2^2 \leq (1 + \delta_r) \|x\|_2^2~$ for all $\|x\|_0 \leq r$.
\end{assumption}

No assumption is made upon $W^{(0)}$.
We hypothesize that Assumption \ref{assm:rip} is mild due to properties like semi or quarter-circle laws that bound the eigenvalues symmetric, random matrices within some range \citep{edelman2013random}, but we leave the formal verification of this assumption as future work.
We define $H = \sigma(W^{(0)}\cdot X) \in \mathbb{R}^{d_{hid} \times B}$, which can be theoretically reformulated as $H = Z + E$ for $s$-row-sparse $Z$ and arbitrary $E$.
From here, we can show the following about the residual between pruned and dense network hidden representations after $t$ iterations of Algorithm \ref{A:cosamp_prune}.

\begin{lemma} \label{L:hidden_res}
If Assumption \ref{assm:rip} holds, the pruned approximation to $H$ after $t$ iterations of Algorithm \ref{A:cosamp_prune}, $H_{\mathcal{S}_t, :}$, is $s$-row-sparse and satisfies the following inequality:
\begin{align*}
    \|\mu(H - H_{\mathcal{S}_t, :})\|_2 \leq (0.444)^t \|\mu(H)\|_2 + \left(14 + \frac{7}{\sqrt{s}}\right)\|\mu(E)\|_1
\end{align*}
\end{lemma}

\noindent
Going further, Lemma \ref{L:hidden_res} can be used to bound the residual between pruned and dense network output. 

\begin{theorem} \label{T:vfrob}
    Let $U = W^{(1)} \cdot H$ and $U' = W_{:, \mathcal{S}_t}^{(1)} \cdot H_{\mathcal{S}_t, :}$ denote pruned and dense network output, respectively. $V_t = U - U'$ stores the residual between pruned and dense network output over the entire dataset. If Assumption \ref{assm:rip} holds, we have the following at iteration $t$ of Algorithm \ref{A:cosamp_prune}:
    \begin{align*}
        \|V_t\|_F \leq \|W^{(1)}\|_F \cdot \left ( (0.444)^t \|\mu(H)\|_2 + \left ( 14 + \frac{7}{\sqrt{s}}\right)\|\mu(E)\|_1 \right)
    \end{align*}
\end{theorem}

Because $\|\mu(H)\|_2$ decays linearly in Theorem \ref{T:vfrob}, $\|\mu(E)\|_1$ dominates the above expression for large $t$.
By assuming $H$ is row-compressible, we can derive a bound on $\|\mu(E)\|_1$.

\begin{lemma} \label{L:noise_bound}
Assume $H$ is $p$-row-compressible with magnitude $R$, where $H = Z + E$ for $s$-row-sparse $Z$ and arbitrary $E$.
Then, 
$\|\mu(E)\|_1 \leq R\cdot \frac{s^{1 - \frac{1}{p}}}{\frac{1}{p} - 1}.$
\end{lemma}

\noindent
Lemma \ref{L:noise_bound} can then be combined with Theorem \ref{T:vfrob} to bound the error due to pruning via i-SpaSP. 

\begin{theorem} \label{T:final_bound}
    Assume Algorithm \ref{A:cosamp_prune} is run for a sufficiently large number of iterations $t$.
    If Assumption \ref{assm:rip} holds and $H$ is $p$-row-compressible with factor $R$, the output residual between the dense network and the pruned network discovered via i-SpaSP can be bounded as follows:
    \begin{align*}
        \|V_t\|_F \leq \mathcal{O}\left ( \frac{s^{\frac{1}{2} - \frac{1}{p}} p (2\sqrt{s} + 1)}{1 - p} \right).
    \end{align*}
    \begin{proof}
    This follows directly from substituting Lemma \ref{L:noise_bound} into Theorem \ref{T:vfrob}, assuming $t$ is large enough such that $(0.444)^t \|W^{(1)}\|_F \|\mu(H)\|_2 \approx 0$, and factoring out constants in the resulting expression.
    \end{proof}
\end{theorem}

Theorem \ref{T:final_bound} indicates that the quality of the pruned network is dependent upon $s$ and $p$.
Intuitively, one would expect that lower values of $p$ (implying sparser $H$) would make pruning easier, as neurons corresponding to zero rows in $H$ could be eliminated without consequence.
This trend is observed exactly within Theorem \ref{T:final_bound}; e.g., for $p = \left \{\frac{3}{4}, \frac{1}{2}, \frac{1}{4} \right \}$ we have $\|V_t\|_F \leq \{\mathcal{O} (s^{-\frac{1}{3}}), \mathcal{O} (s^{-1} ), \mathcal{O}(s^{-3} )  \}$, respectively.
\textit{To the best of the authors' knowledge, our work is the first to theoretically characterize pruning error with respect to sparsity properties of network hidden representations.}
This bound can also be extended to similarly-structured (see Appendix \ref{A:multi_layer}), multi-layer networks: 

\begin{theorem} \label{T:multi_layer}
Consider an $L$-hidden-layer network with weight matrices $\{W^{(0)}, \dots, W^{(L)} \}$ and hidden representations $\{H^{(1)}, \dots H^{(L)} \}$.
The hidden representations of each layer are assumed to have dimension $d$ for simplicity. 
We define $H^{(\ell)} = \sigma (W^{(\ell)} \cdot H^{(\ell - 1)})$, where $H^{(1)} = \sigma(W^{(0)} \cdot X)$ and $H^{(L)} = W^{(L)} \cdot H^{(L - 1)}$.
We assume all weight matrices other than $W^{(0)}$ obey Assumption \ref{assm:rip} and all hidden representations other than $H^{(L)}$ are $p$-row-compressible.
i-SpaSP is applied greedily to prune each network layer, in layer order, from $d$ to $s$ hidden neurons.
Given sufficient iterations $t$, the residual between pruned and dense multi-layer network output behaves as:
\begin{align}
\|V^{(L)}_t\|_F \leq \mathcal{O} \left ( \sum_{i=1}^L  \textcolor{mygreen}{\left (14 + \frac{7}{\sqrt{s}}\right)^{L - i + 1}}  \textcolor{myblue2}{\left (\|W^{(L)}\|_F   \prod_{j=1}^{L - i} \|\texttt{vec}(W^{(j)})\|_1 \right)} \textcolor{myred}{\left (\frac{ d^{\frac{L-i}{2}} s^{1 - \frac{1}{p}}}{ \frac{1}{p} - 1} \right)} \right ) \label{eq:multi_layer}
\end{align}
\end{theorem}

Pruning error in \eqref{eq:multi_layer} is summed over each network layer.
Considering layer $i$, the green factor is inherited from Theorem \ref{T:vfrob} with an added exponent due to a recursion over network layers after $i$.
Similarly, the blue factor accounts for propagation of error through weight matrices after layer $i$, \emph{revealing that green and blue factors account for propagation of error through network layers}.
The red portion of \eqref{eq:multi_layer}, which captures the convergence properties of multi-layer pruning, comes from Lemma \ref{L:noise_bound}, where an extra factor $d^{\frac{L - i}{2}}$ arises as an artifact of the proof.\footnote{This artifact arises because $\ell_1$ norms must be replaced with $\ell_2$ norms in certain areas to form a recursion over network layers. This artifact can likely be removed by leveraging sparse matrix analysis for expander graphs \citep{bah2018construction}, but we leave this as future work to simplify the analysis.}
Because $d$ is fixed multiple of $s$ determined by the pruning ratio, this factor disappears when $p$ is small, leading the expression to converge.
For example, if $p=\frac{1}{4}$, the red expression in \eqref{eq:multi_layer} behaves asymptotically as $\{\mathcal{O}(s^{-2.5}), \mathcal{O}(s^{-2}), \mathcal{O}(s^{-1.5}), \dots \}$ for layers at the end of the network moving backwards.

\section{Experiments} \label{S:experiments}

\begin{wrapfigure}{r}{0.6\textwidth}
    \vspace{-0.6cm}
    \centering
    \includegraphics[width=1\linewidth]{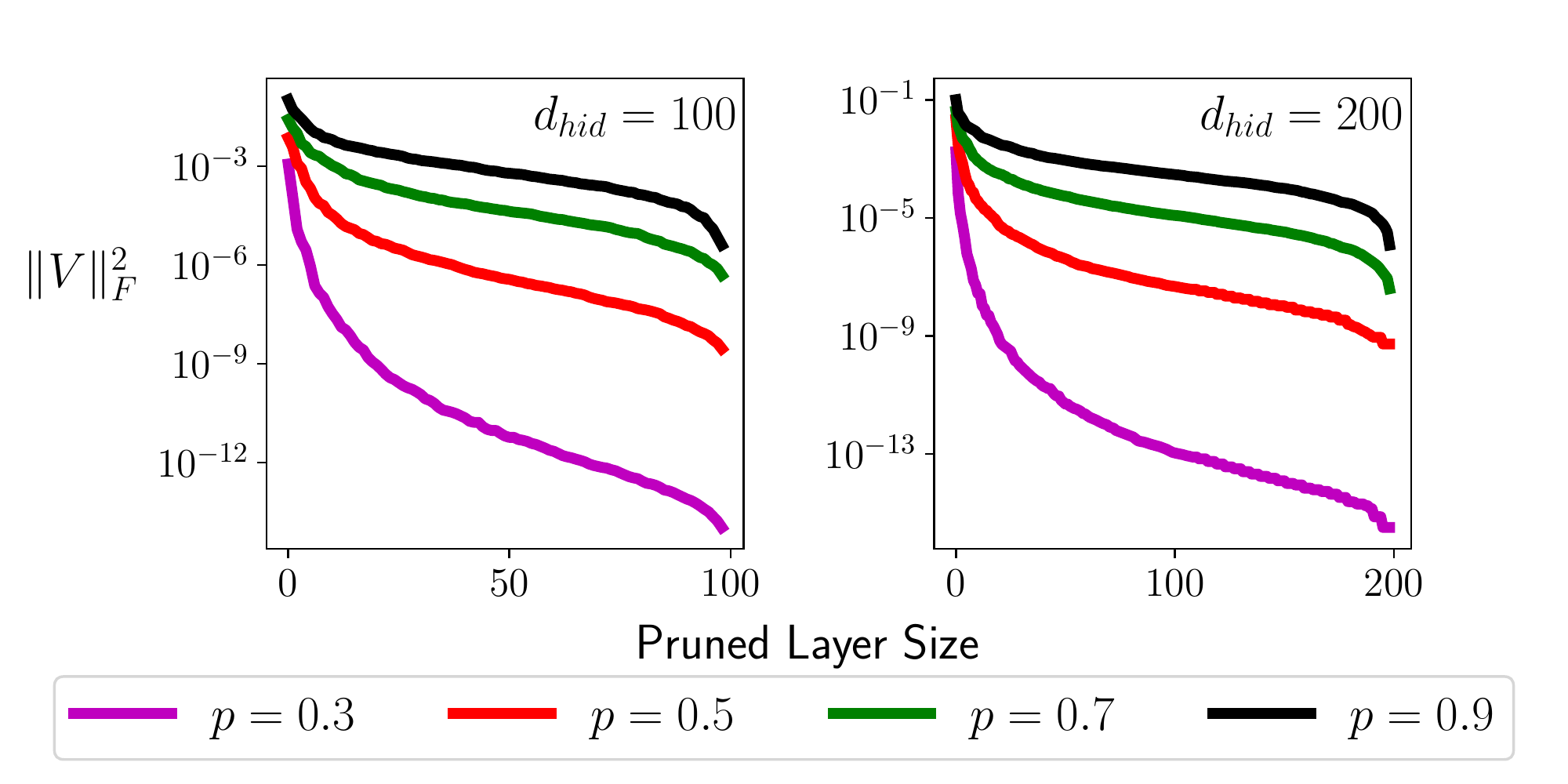}
    \caption{i-SpaSP pruning experiments for two-layer networks of different sizes and $p$ ratios. \emph{The error decay rate increases with decreasing $p$, confirming that more compressible hidden representations aid in the pruning process.}}
    \label{fig:synthetic} \vspace{-0.4cm}
\end{wrapfigure}

Within this section, we provide empirical analysis of i-SpaSP.
We first present synthetic results using two-layer neural networks to numerically verify Theorem \ref{T:final_bound}. 
Then, we perform experiments with two-layer networks on MNIST \citep{mnist}, convolutional neural networks (CNNs) on ImageNet (ILSVRC2012), and multi-lingual BERT (mBERT) \citep{devlin2018bert} on the cross-lingual NLI corpus (XNLI) \citep{conneau2018xnli}.
For all experiments, we adopt best practices from previous work \citep{conv_filt_prune} to determine pruning ratios within the dense network, often performing less pruning on sensitive layers; see Appendix \ref{A:exp_details} for details.
As baselines, we adopt both greedy selection methodologies \citep{provable_subnetworks, log_pruning} and several common, heuristic methods.
We find that, in addition to improving upon the pruning efficiency of GFS (i.e., see Section \ref{S:comp_complex} for runtime comparisons on ResNet pruning experiments), i-SpaSP performs comparably to baselines in all cases, demonstrating that it is both performant and efficient.



\subsection{Synthetic Experiments} \label{S:synthetic_exp}
To numerically verify Theorem \ref{T:final_bound}, we construct synthetic $H$ matrices with different row-compressibility ratios $p$, denoted as $\mathcal{H} = \{H^{(i)}\}_{i=1}^K$.
For each $p \in \{0.3, 0.5, 0.7, 0.9\}$, we randomly generate three unique entries within $\mathcal{H}$ (i.e., $K=12$) and present average performance across them all.
We prune a randomly-initialized $W^{(1)}$ from size $d_{hid} \times d_{out}$ to sizes $s \times d_{out}$ for $s \in [d_{hid}]$. (i.e., each setting of $s$ is a separate experiment) using $T = 20$.\footnote{We find that the active set of neurons selected by i-SpaSP becomes stable (i.e., few neurons are modified) after $20$ iterations or less in all synthetic experiments.}.
We test $d_{hid} \in \{100, 200\}$, and the same $W^{(1)}$ matrix is used for experiments with equal hidden dimension; see Appendix \ref{A:synthetic} for more details.


Results are displayed in Figure \ref{fig:synthetic}, where  $\|V\|_F^2$ is shown to decay polynomially with respect to the number of neurons within the pruned network $s$.
Furthermore, as predicted by Theorem \ref{T:final_bound}, the decay rate of $\|V\|^2_F$ increases as $p$ decreases, revealing that higher levels of sparsity within the dense network's hidden representations improve pruning performance and speed with respect to $s$. 
This trend holds for all hidden dimensions that were considered.

\subsection{Two-Layer Networks} \label{S:two_layer}
\begin{wrapfigure}{r}{0.6\textwidth}
    \vspace{-0.4cm}
    \centering
    \includegraphics[width=1\linewidth]{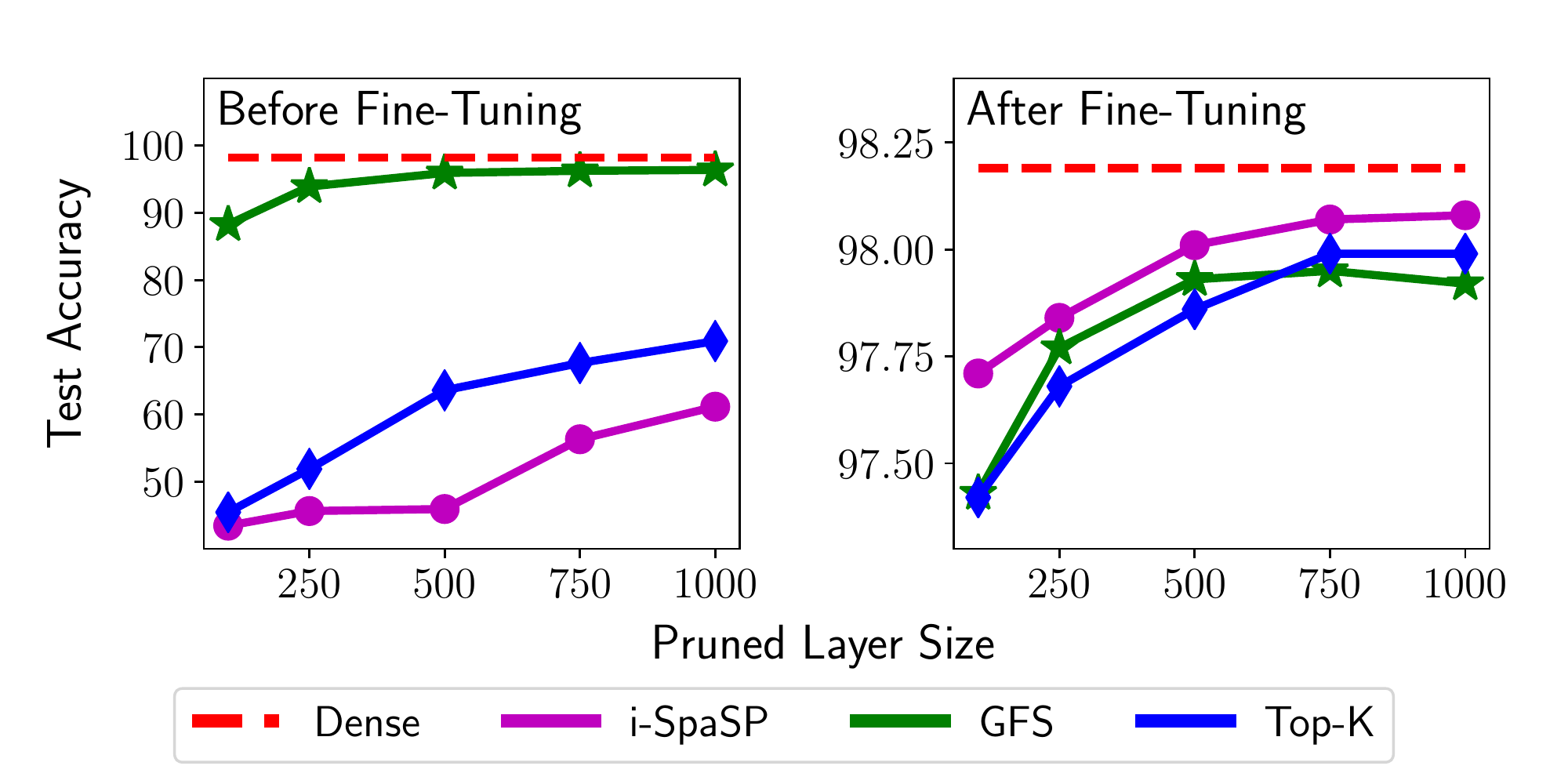}
    \caption{Performance of networks pruned with different greedy algorithms on MNIST before (left) and after (right) fine-tuning. Although GFS performs well prior to fine-tuning, \emph{i-SpaSP always yields the top-performing network after fine-tuning.}}
    \label{fig:mnist_prune} \vspace{-0.2cm}
\end{wrapfigure}

We perform pruning experiments with two-layer networks on MNIST \citep{mnist}.
All MNIST images are flattened and no data augmentation is used.
The dense network has $10 \times 10^{3}$ hidden neurons and is pre-trained before pruning.
We prune the dense network using i-SpaSP, GFS, and Top-K, which we use as a naive greedy selection baseline.\footnote{Top-k selects $k$ neurons with the largest-magnitude hidden representations in a mini-batch of data. It is a naive baseline for greedy selection that is not used in previous work to the best of our knowledge.}
After pruning, the network is fine-tuned using stochastic gradient descent (SGD) with momentum.
Performance is reported as an average across three separate trails; see Appendix \ref{A:two_layer} for more details. 

As shown in Figure \ref{fig:mnist_prune}-right, i-SpaSP outperforms other pruning methodologies after fine-tuning in all experimental settings.
Networks obtained with GFS perform well without fine-tuning (i.e., Figure \ref{fig:mnist_prune}-left) because neurons are selected to minimize loss during the pruning process.
Because i-SpaSP selects neurons based upon the importance criteria described in Section \ref{S:method}, the pruning process does not directly minimize training loss, thus leading to poorer performance prior to fine-tuning (i.e., Top-K exhibits similar behavior).
Nonetheless, i-SpaSP, in addition to improving upon the pruning efficiency of GFS, discovers a set of neurons that more closely recovers dense network output, as revealed by its superior performance after fine-tuning.

\subsection{Deep Convolutional Networks} \label{S:deep_net}
We perform structured filter pruning of ResNet34 \citep{resnet} and MobileNetV2 \citep{mobilenetv2} with i-SpaSP on ImageNet (ILSVRC 2012).
Beginning with public, pre-trained models \citep{paszke2017automatic}, we use i-SpaSP to prune chosen convolutional blocks within each network, then fine-tune the pruned model using SGD with momentum.
After pruning each block, we perform a small amount of fine-tuning; see Appendix \ref{A:deep_net} for more details.
Numerous heuristic and greedy selection-based algorithms are adopted as baselines; see Table \ref{tab:imagenet}.

\begin{wraptable}{r}{0.6\textwidth}
    \centering
    \begin{footnotesize}
    \begin{tabular}{c|c|cc}
    \toprule
          & Pruning Method & Top-1 Acc. & FLOPS \\
         \midrule
         \multirow{12}{*}{\rotatebox{90}{ResNet34}} & Full Model \cite{resnet} & 73.4 & 3.68G\\
         \cmidrule{2-4}
         & Filter Pruning \cite{conv_filt_prune} & 72.1 & 2.79G\\
         & Rethinking Pruning \cite{rethink_pruning} & 72.0 & 2.79G \\
         & More is Less \cite{dong2017more} & 73.0 & 2.75G \\
         & i-SpaSP & 73.5 & 2.69G\\
         & GFS \cite{provable_subnetworks} & 73.5 & 2.64G \\
         & Multi-Layer GFS \cite{log_pruning} & 73.5 & 2.20G\\
         \cmidrule{2-4}
         & SFP \cite{he2018soft} & 71.8 & 2.17G \\
         & FPGM \cite{he2019filter} & 72.5 & 2.16G \\
         & i-SpaSP & 72.5 & 2.13G\\
         & GFS \cite{provable_subnetworks} & 72.9 & 2.07G \\
         & Multi-Layer GFS \cite{log_pruning} & 73.3 & 1.90G \\
         \midrule
         \multirow{12}{*}{\rotatebox{90}{MobileNetV2}} & Full Model \cite{mobilenetv2} & 72.0 & 314M \\
         \cmidrule{2-4}
         & i-SpaSP & 71.6 & 260M\\
         & GFS \cite{provable_subnetworks} & 71.9 & 258M\\
         & Multi-Layer GFS \cite{log_pruning} & 72.2 & 245M\\
         & i-SpaSP & 71.3 & 242M\\
         \cmidrule{2-4}
         & LEGR \cite{chin2019legr} & 71.4 & 224M\\
         & Uniform & 70.0 & 220M \\
         & AMC \cite{he2018amc} & 70.8 & 220M\\
         & i-SpaSP & 70.7 & 220M \\
         & GFS \cite{provable_subnetworks} & 71.6 & 220M \\
         & Multi-Layer GFS \cite{log_pruning} & 71.7 & 218M \\
         & Meta Pruning \cite{liu2019metapruning} & 71.2 & 217M\\
    \bottomrule
    \end{tabular}
    \caption{Test accuracy of ResNet34 and MobileNetV2 models pruned to different FLOP levels with various pruning algorithms on ImageNet.}
        \label{tab:imagenet} \vspace{-0.5cm}
    \end{footnotesize}
\end{wraptable}

\medskip
\noindent \textbf{ResNet34.} 
We prune ResNet34 to 2.69 and 2.13 GFlops with i-SpaSP.
As shown in Table \ref{tab:imagenet}, i-SpaSP yields comparable performance to GFS variants at similar FLOP levels; e.g., i-SpaSP matches 75.5\% test accuracy of GFS with the 2.69 GFlop model and performs within 1\% of GFS variants for the 2.13 GFlop model.
However, the multi-layer variant of GFS \citep{log_pruning} does discover sub-networks with fewer FLOPS and similar performance in both cases.
i-SpaSP improves upon or matches the performance of all heuristic methods at similar FLOP levels.


\medskip
\noindent \textbf{MobileNetV2.} 
We prune MobileNetV2 to 260, 242, and 220 MFlops with i-SpaSP.
In all cases, sub-networks discovered with i-SpaSP achieve performance within 1\% of those obtained via both GFS variants and heuristic methods.
Although MobileNetV2 performance is relatively lower in comparison to ResNet34, i-SpaSP is still capable of pruning the more difficult network to various different FLOP levels and performs comparably to baselines in all cases.

\medskip
\noindent \textbf{Discussion.} Within Table \ref{tab:imagenet}, the performance of i-SpaSP is never more than 1\% below that of a similar-FLOP model obtained with a baseline pruning methodology, including both greedy selection and heuristic-based methods.
As such, similar to GFS, i-SpaSP can be seen as a theoretically-grounded pruning methodology that is practically useful, even in large-scale experiments.
Such pruning methodologies that are both theoretically and practically relevant are few.
In comparison to GFS variants, i-SpaSP significantly improves pruning efficiency; see Section \ref{S:comp_complex}.
Thus, i-SpaSP can be seen as a viable alternative to GFS---both theoretically and practically---that may be preferable when runtime is a major concern. 


\subsection{Transformer Networks} \label{S:transformers}

\begin{wraptable}{r}{0.5\textwidth}
    \centering
    \begin{footnotesize}
    \begin{tabular}{c|cc}
    \toprule
          Method & Pruning Ratio & Top-1 Acc. \\
          \midrule
          Full Model & - & 71.02\\
          \midrule \multirow{2}{*}{Uniform} & 25\% & 62.73 \\
          & 40\% & 66.43 \\
          \midrule \multirow{2}{*}{\cite{michel2019sixteen}} & 25\% & 62.97 \\
          & 40\% & 67.31 \\
          \midrule \multirow{2}{*}{i-SpaSP} & 25\% & 63.70\\
          & 40\% & 68.11\\
    \bottomrule
    \end{tabular}
    \caption{Test accuracy of mBERT models pruned and fine-tuned on the XNLI dataset.}
        \label{tab:xnli} \vspace{-0.4cm}
    \end{footnotesize}
\end{wraptable}

We perform structured pruning of attention heads using the mBERT model \citep{devlin2018bert} on the XNLI dataset \citep{conneau2018xnli}, which contains textual entailment annotations across 15 different languages.
We begin with a pre-trained mBERT model and fine-tune it on XNLI prior to pruning.
Then, structured pruning is performed on the attention heads of each layer using i-SpaSP, uniform pruning (i.e., randomly removing a fixed ratio of attention heads), and a sensitivity-based masking approach \citep{michel2019sixteen} (i.e., a state-of-the-art heuristic approach for structured attention head pruning for transformers).
After pruning models to a fixed ratio of 25\% or 40\% of original attention heads, we fine-tune the pruned networks and record their performance on the XNLI test set; see Appendix \ref{A:transformers} for further details. 


As shown in Table \ref{tab:xnli}, models pruned with i-SpaSP outperform those pruned to the same ratio with either uniform or heuristic pruning methods in all cases. 
The ability of i-SpaSP to effectively prune transformers demonstrates that the methodology can be applied to the structured pruning of numerous different architectures without significant implementation changes (i.e., using the automatic differentiation approach described in Section \ref{S:implementation}).
Furthermore, i-SpaSP even outperforms a state-of-the-art heuristic approach for the structured pruning of transformer attention heads, thus again highlighting the strong empirical performance of i-SpaSP in comparison to heuristic pruning methods that lack theoretical guarantees.

\section{Conclusion}
We propose i-SpaSP, a pruning methodology for neural networks inspired by sparse signal recovery.
Our methodology comes with theoretical guarantees that indicate, for both two and multi-layer networks, the quality of a pruned network decays polynomially with respect to its size.
We connect this theoretical analysis to properties of the dense network, showing that pruning performance improves as the dense network's hidden representations become more sparse.
Practically, i-SpaSP performs comparably to numerous baseline pruning methodologies in large-scale experiments and drastically improves upon the computationally efficiency of the most common provable pruning methodologies.
As such, i-SpaSP is a practical, provable, and efficient algorithm that we hope will enable a better understanding of neural network pruning both in theory and practice. 
In future work, we wish to extend i-SpaSP to cover cases in which network pruning and training are combined into a single process, such as utilizing regularization-based approaches to induce sparsity during pre-training or updating network weights to improve sub-network performance as pruning occurs.

\clearpage

\acks{Compute resources used for this work were provided by Alegion Inc. This work is supported by NSF FET:Small no. 1907936, NSF MLWiNS CNS no. 2003137 (in collaboration with Intel), NSF CMMI no. 2037545, NSF CAREER award no. 2145629, and Rice InterDisciplinary Excellence Award (IDEA).}

\bibliography{biblio}

\begin{thebibliography}{58}
\providecommand{\natexlab}[1]{#1}
\providecommand{\url}[1]{\texttt{#1}}
\expandafter\ifx\csname urlstyle\endcsname\relax
  \providecommand{\doi}[1]{doi: #1}\else
  \providecommand{\doi}{doi: \begingroup \urlstyle{rm}\Url}\fi

\bibitem[Abadi et~al.(2015)Abadi, Agarwal, Barham, Brevdo, Chen, Citro,
  Corrado, Davis, Dean, Devin, Ghemawat, Goodfellow, Harp, Irving, Isard, Jia,
  Jozefowicz, Kaiser, Kudlur, Levenberg, Man\'{e}, Monga, Moore, Murray, Olah,
  Schuster, Shlens, Steiner, Sutskever, Talwar, Tucker, Vanhoucke, Vasudevan,
  Vi\'{e}gas, Vinyals, Warden, Wattenberg, Wicke, Yu, and Zheng]{tensorflow}
Mart\'{\i}n Abadi, Ashish Agarwal, Paul Barham, Eugene Brevdo, Zhifeng Chen,
  Craig Citro, Greg~S. Corrado, Andy Davis, Jeffrey Dean, Matthieu Devin,
  Sanjay Ghemawat, Ian Goodfellow, Andrew Harp, Geoffrey Irving, Michael Isard,
  Yangqing Jia, Rafal Jozefowicz, Lukasz Kaiser, Manjunath Kudlur, Josh
  Levenberg, Dan Man\'{e}, Rajat Monga, Sherry Moore, Derek Murray, Chris Olah,
  Mike Schuster, Jonathon Shlens, Benoit Steiner, Ilya Sutskever, Kunal Talwar,
  Paul Tucker, Vincent Vanhoucke, Vijay Vasudevan, Fernanda Vi\'{e}gas, Oriol
  Vinyals, Pete Warden, Martin Wattenberg, Martin Wicke, Yuan Yu, and Xiaoqiang
  Zheng.
\newblock {TensorFlow}: Large-scale machine learning on heterogeneous systems,
  2015.

\bibitem[{Bach}(2014)]{bach_fw}
Francis {Bach}.
\newblock {Breaking the Curse of Dimensionality with Convex Neural Networks}.
\newblock \emph{arXiv e-prints}, art. arXiv:1412.8690, December 2014.

\bibitem[Bah and Tanner(2018)]{bah2018construction}
Bubacarr Bah and Jared Tanner.
\newblock On the construction of sparse matrices from expander graphs.
\newblock \emph{Frontiers in Applied Mathematics and Statistics}, 4:\penalty0
  39, 2018.

\bibitem[Baykal et~al.(2018)Baykal, Liebenwein, Gilitschenski, Feldman, and
  Rus]{baykal2018data}
Cenk Baykal, Lucas Liebenwein, Igor Gilitschenski, Dan Feldman, and Daniela
  Rus.
\newblock Data-dependent coresets for compressing neural networks with
  applications to generalization bounds.
\newblock \emph{arXiv preprint arXiv:1804.05345}, 2018.

\bibitem[{Baykal} et~al.(2019){Baykal}, {Liebenwein}, {Gilitschenski},
  {Feldman}, and {Rus}]{sipping_nn}
Cenk {Baykal}, Lucas {Liebenwein}, Igor {Gilitschenski}, Dan {Feldman}, and
  Daniela {Rus}.
\newblock {SiPPing Neural Networks: Sensitivity-informed Provable Pruning of
  Neural Networks}.
\newblock \emph{arXiv e-prints}, art. arXiv:1910.05422, October 2019.

\bibitem[Candes and Tao(2006)]{rip}
Emmanuel~J Candes and Terence Tao.
\newblock Near-optimal signal recovery from random projections: Universal
  encoding strategies?
\newblock \emph{IEEE transactions on information theory}, 52\penalty0
  (12):\penalty0 5406--5425, 2006.

\bibitem[{Chen} et~al.(2020){Chen}, {Cheng}, {Wang}, {Gan}, {Wang}, and
  {Liu}]{early_bert}
Xiaohan {Chen}, Yu~{Cheng}, Shuohang {Wang}, Zhe {Gan}, Zhangyang {Wang}, and
  Jingjing {Liu}.
\newblock {EarlyBERT: Efficient BERT Training via Early-bird Lottery Tickets}.
\newblock \emph{arXiv e-prints}, art. arXiv:2101.00063, December 2020.

\bibitem[Chin et~al.(2019)Chin, Ding, Zhang, and Marculescu]{chin2019legr}
Ting-Wu Chin, Ruizhou Ding, Cha Zhang, and Diana Marculescu.
\newblock Legr: Filter pruning via learned global ranking.
\newblock 2019.

\bibitem[Conneau et~al.(2018)Conneau, Rinott, Lample, Williams, Bowman,
  Schwenk, and Stoyanov]{conneau2018xnli}
Alexis Conneau, Ruty Rinott, Guillaume Lample, Adina Williams, Samuel~R.
  Bowman, Holger Schwenk, and Veselin Stoyanov.
\newblock Xnli: Evaluating cross-lingual sentence representations.
\newblock In \emph{Proceedings of the 2018 Conference on Empirical Methods in
  Natural Language Processing}. Association for Computational Linguistics,
  2018.

\bibitem[Deng(2012)]{mnist}
Li~Deng.
\newblock The mnist database of handwritten digit images for machine learning
  research.
\newblock \emph{IEEE Signal Processing Magazine}, 29\penalty0 (6):\penalty0
  141--142, 2012.

\bibitem[Devlin et~al.(2018)Devlin, Chang, Lee, and Toutanova]{devlin2018bert}
Jacob Devlin, Ming-Wei Chang, Kenton Lee, and Kristina Toutanova.
\newblock Bert: Pre-training of deep bidirectional transformers for language
  understanding.
\newblock \emph{arXiv preprint arXiv:1810.04805}, 2018.

\bibitem[Dong et~al.(2017)Dong, Huang, Yang, and Yan]{dong2017more}
Xuanyi Dong, Junshi Huang, Yi~Yang, and Shuicheng Yan.
\newblock More is less: A more complicated network with less inference
  complexity.
\newblock In \emph{Proceedings of the IEEE Conference on Computer Vision and
  Pattern Recognition}, pages 5840--5848, 2017.

\bibitem[Edelman and Wang(2013)]{edelman2013random}
Alan Edelman and Yuyang Wang.
\newblock Random matrix theory and its innovative applications.
\newblock In \emph{Advances in Applied Mathematics, Modeling, and Computational
  Science}, pages 91--116. Springer, 2013.

\bibitem[{Evci} et~al.(2019){Evci}, {Gale}, {Menick}, {Castro}, and
  {Elsen}]{rigl}
Utku {Evci}, Trevor {Gale}, Jacob {Menick}, Pablo~Samuel {Castro}, and Erich
  {Elsen}.
\newblock {Rigging the Lottery: Making All Tickets Winners}.
\newblock \emph{arXiv e-prints}, art. arXiv:1911.11134, November 2019.

\bibitem[{Evci} et~al.(2020){Evci}, {Ioannou}, {Keskin}, and
  {Dauphin}]{how_lth_wins}
Utku {Evci}, Yani~A. {Ioannou}, Cem {Keskin}, and Yann {Dauphin}.
\newblock {Gradient Flow in Sparse Neural Networks and How Lottery Tickets
  Win}.
\newblock \emph{arXiv e-prints}, art. arXiv:2010.03533, October 2020.

\bibitem[Frank et~al.(1956)Frank, Wolfe, et~al.]{frank1956algorithm}
Marguerite Frank, Philip Wolfe, et~al.
\newblock An algorithm for quadratic programming.
\newblock \emph{Naval research logistics quarterly}, 3\penalty0 (1-2):\penalty0
  95--110, 1956.

\bibitem[{Frankle} and {Carbin}(2018)]{lth}
Jonathan {Frankle} and Michael {Carbin}.
\newblock {The Lottery Ticket Hypothesis: Finding Sparse, Trainable Neural
  Networks}.
\newblock \emph{arXiv e-prints}, art. arXiv:1803.03635, March 2018.

\bibitem[{Frankle} et~al.(2019){Frankle}, {Karolina Dziugaite}, {Roy}, and
  {Carbin}]{stable_lth}
Jonathan {Frankle}, Gintare {Karolina Dziugaite}, Daniel~M. {Roy}, and Michael
  {Carbin}.
\newblock {Stabilizing the Lottery Ticket Hypothesis}.
\newblock \emph{arXiv e-prints}, art. arXiv:1903.01611, March 2019.

\bibitem[{Han} et~al.(2015){Han}, {Mao}, and {Dally}]{deep_comp}
Song {Han}, Huizi {Mao}, and William~J. {Dally}.
\newblock {Deep Compression: Compressing Deep Neural Networks with Pruning,
  Trained Quantization and Huffman Coding}.
\newblock \emph{arXiv e-prints}, art. arXiv:1510.00149, October 2015.

\bibitem[Han et~al.(2015{\natexlab{a}})Han, Mao, and Dally]{han2015deep}
Song Han, Huizi Mao, and William~J Dally.
\newblock Deep compression: Compressing deep neural networks with pruning,
  trained quantization and huffman coding.
\newblock \emph{arXiv preprint arXiv:1510.00149}, 2015{\natexlab{a}}.

\bibitem[Han et~al.(2015{\natexlab{b}})Han, Pool, Tran, and
  Dally]{han2015learning}
Song Han, Jeff Pool, John Tran, and William~J Dally.
\newblock Learning both weights and connections for efficient neural networks.
\newblock \emph{arXiv preprint arXiv:1506.02626}, 2015{\natexlab{b}}.

\bibitem[{Han} et~al.(2016){Han}, {Liu}, {Mao}, {Pu}, {Pedram}, {Horowitz}, and
  {Dally}]{eie}
Song {Han}, Xingyu {Liu}, Huizi {Mao}, Jing {Pu}, Ardavan {Pedram}, Mark~A.
  {Horowitz}, and William~J. {Dally}.
\newblock {EIE: Efficient Inference Engine on Compressed Deep Neural Network}.
\newblock \emph{arXiv e-prints}, art. arXiv:1602.01528, February 2016.

\bibitem[{He} et~al.(2015{\natexlab{a}}){He}, {Zhang}, {Ren}, and
  {Sun}]{2015arXiv150201852H}
Kaiming {He}, Xiangyu {Zhang}, Shaoqing {Ren}, and Jian {Sun}.
\newblock {Delving Deep into Rectifiers: Surpassing Human-Level Performance on
  ImageNet Classification}.
\newblock \emph{arXiv e-prints}, art. arXiv:1502.01852, February
  2015{\natexlab{a}}.

\bibitem[{He} et~al.(2015{\natexlab{b}}){He}, {Zhang}, {Ren}, and
  {Sun}]{resnet}
Kaiming {He}, Xiangyu {Zhang}, Shaoqing {Ren}, and Jian {Sun}.
\newblock {Deep Residual Learning for Image Recognition}.
\newblock \emph{arXiv e-prints}, art. arXiv:1512.03385, December
  2015{\natexlab{b}}.

\bibitem[He et~al.(2018{\natexlab{a}})He, Kang, Dong, Fu, and Yang]{he2018soft}
Yang He, Guoliang Kang, Xuanyi Dong, Yanwei Fu, and Yi~Yang.
\newblock Soft filter pruning for accelerating deep convolutional neural
  networks.
\newblock \emph{arXiv preprint arXiv:1808.06866}, 2018{\natexlab{a}}.

\bibitem[He et~al.(2019)He, Liu, Wang, Hu, and Yang]{he2019filter}
Yang He, Ping Liu, Ziwei Wang, Zhilan Hu, and Yi~Yang.
\newblock Filter pruning via geometric median for deep convolutional neural
  networks acceleration.
\newblock In \emph{Proceedings of the IEEE/CVF Conference on Computer Vision
  and Pattern Recognition}, pages 4340--4349, 2019.

\bibitem[{He} et~al.(2017){He}, {Zhang}, and {Sun}]{he_pruning}
Yihui {He}, Xiangyu {Zhang}, and Jian {Sun}.
\newblock {Channel Pruning for Accelerating Very Deep Neural Networks}.
\newblock \emph{arXiv e-prints}, art. arXiv:1707.06168, July 2017.

\bibitem[He et~al.(2018{\natexlab{b}})He, Lin, Liu, Wang, Li, and
  Han]{he2018amc}
Yihui He, Ji~Lin, Zhijian Liu, Hanrui Wang, Li-Jia Li, and Song Han.
\newblock Amc: Automl for model compression and acceleration on mobile devices.
\newblock In \emph{Proceedings of the European conference on computer vision
  (ECCV)}, pages 784--800, 2018{\natexlab{b}}.

\bibitem[Huang and Wang(2018)]{huang2018data}
Zehao Huang and Naiyan Wang.
\newblock Data-driven sparse structure selection for deep neural networks.
\newblock In \emph{Proceedings of the European conference on computer vision
  (ECCV)}, pages 304--320, 2018.

\bibitem[Khanna and Kyrillidis(2018)]{khanna2018iht}
Rajiv Khanna and Anastasios Kyrillidis.
\newblock Iht dies hard: Provable accelerated iterative hard thresholding.
\newblock In \emph{International Conference on Artificial Intelligence and
  Statistics}, pages 188--198. PMLR, 2018.

\bibitem[{Li} et~al.(2016){Li}, {Kadav}, {Durdanovic}, {Samet}, and
  {Graf}]{conv_filt_prune}
Hao {Li}, Asim {Kadav}, Igor {Durdanovic}, Hanan {Samet}, and Hans~Peter
  {Graf}.
\newblock {Pruning Filters for Efficient ConvNets}.
\newblock \emph{arXiv e-prints}, art. arXiv:1608.08710, August 2016.

\bibitem[Liebenwein et~al.(2019)Liebenwein, Baykal, Lang, Feldman, and
  Rus]{liebenwein2019provable}
Lucas Liebenwein, Cenk Baykal, Harry Lang, Dan Feldman, and Daniela Rus.
\newblock Provable filter pruning for efficient neural networks.
\newblock \emph{arXiv preprint arXiv:1911.07412}, 2019.

\bibitem[Liu et~al.(2019)Liu, Mu, Zhang, Guo, Yang, Cheng, and
  Sun]{liu2019metapruning}
Zechun Liu, Haoyuan Mu, Xiangyu Zhang, Zichao Guo, Xin Yang, Kwang-Ting Cheng,
  and Jian Sun.
\newblock Metapruning: Meta learning for automatic neural network channel
  pruning.
\newblock In \emph{Proceedings of the IEEE/CVF International Conference on
  Computer Vision}, pages 3296--3305, 2019.

\bibitem[{Liu} et~al.(2017){Liu}, {Li}, {Shen}, {Huang}, {Yan}, and
  {Zhang}]{net_slimming}
Zhuang {Liu}, Jianguo {Li}, Zhiqiang {Shen}, Gao {Huang}, Shoumeng {Yan}, and
  Changshui {Zhang}.
\newblock {Learning Efficient Convolutional Networks through Network Slimming}.
\newblock \emph{arXiv e-prints}, art. arXiv:1708.06519, August 2017.

\bibitem[{Liu} et~al.(2018){Liu}, {Sun}, {Zhou}, {Huang}, and
  {Darrell}]{rethink_pruning}
Zhuang {Liu}, Mingjie {Sun}, Tinghui {Zhou}, Gao {Huang}, and Trevor {Darrell}.
\newblock {Rethinking the Value of Network Pruning}.
\newblock \emph{arXiv e-prints}, art. arXiv:1810.05270, October 2018.

\bibitem[{Luo} et~al.(2017){Luo}, {Wu}, and {Lin}]{thinet}
Jian-Hao {Luo}, Jianxin {Wu}, and Weiyao {Lin}.
\newblock {ThiNet: A Filter Level Pruning Method for Deep Neural Network
  Compression}.
\newblock \emph{arXiv e-prints}, art. arXiv:1707.06342, July 2017.

\bibitem[{Malach} et~al.(2020){Malach}, {Yehudai}, {Shalev-Shwartz}, and
  {Shamir}]{pruning_is_all_you_need}
Eran {Malach}, Gilad {Yehudai}, Shai {Shalev-Shwartz}, and Ohad {Shamir}.
\newblock {Proving the Lottery Ticket Hypothesis: Pruning is All You Need}.
\newblock \emph{arXiv e-prints}, art. arXiv:2002.00585, February 2020.

\bibitem[Michel et~al.(2019)Michel, Levy, and Neubig]{michel2019sixteen}
Paul Michel, Omer Levy, and Graham Neubig.
\newblock Are sixteen heads really better than one?
\newblock \emph{Advances in neural information processing systems}, 32, 2019.

\bibitem[Molchanov et~al.(2016)Molchanov, Tyree, Karras, Aila, and
  Kautz]{molchanov2016pruning}
Pavlo Molchanov, Stephen Tyree, Tero Karras, Timo Aila, and Jan Kautz.
\newblock Pruning convolutional neural networks for resource efficient
  inference.
\newblock \emph{arXiv preprint arXiv:1611.06440}, 2016.

\bibitem[Mussay et~al.(2019)Mussay, Osadchy, Braverman, Zhou, and
  Feldman]{mussay2019data}
Ben Mussay, Margarita Osadchy, Vladimir Braverman, Samson Zhou, and Dan
  Feldman.
\newblock Data-independent neural pruning via coresets.
\newblock \emph{arXiv preprint arXiv:1907.04018}, 2019.

\bibitem[Needell and Tropp(2009)]{cosamp}
Deanna Needell and Joel~A Tropp.
\newblock Cosamp: Iterative signal recovery from incomplete and inaccurate
  samples.
\newblock \emph{Applied and computational harmonic analysis}, 26\penalty0
  (3):\penalty0 301--321, 2009.

\bibitem[Nemhauser et~al.(1978)Nemhauser, Wolsey, and Fisher]{submod_opt}
George~L Nemhauser, Laurence~A Wolsey, and Marshall~L Fisher.
\newblock An analysis of approximations for maximizing submodular set
  functions—i.
\newblock \emph{Mathematical programming}, 14\penalty0 (1):\penalty0 265--294,
  1978.

\bibitem[{Orseau} et~al.(2020){Orseau}, {Hutter}, and
  {Rivasplata}]{log_pruning_is_all_you_need}
Laurent {Orseau}, Marcus {Hutter}, and Omar {Rivasplata}.
\newblock {Logarithmic Pruning is All You Need}.
\newblock \emph{arXiv e-prints}, art. arXiv:2006.12156, June 2020.

\bibitem[Paszke et~al.(2017)Paszke, Gross, Chintala, Chanan, Yang, DeVito, Lin,
  Desmaison, Antiga, and Lerer]{paszke2017automatic}
Adam Paszke, Sam Gross, Soumith Chintala, Gregory Chanan, Edward Yang, Zachary
  DeVito, Zeming Lin, Alban Desmaison, Luca Antiga, and Adam Lerer.
\newblock Automatic differentiation in pytorch.
\newblock 2017.

\bibitem[Pensia et~al.(2020)Pensia, Rajput, Nagle, Vishwakarma, and
  Papailiopoulos]{subset_lth}
Ankit Pensia, Shashank Rajput, Alliot Nagle, Harit Vishwakarma, and Dimitris
  Papailiopoulos.
\newblock Optimal lottery tickets via subsetsum: Logarithmic
  over-parameterization is sufficient.
\newblock \emph{arXiv preprint arXiv:2006.07990}, 2020.

\bibitem[{Pokutta} et~al.(2020){Pokutta}, {Spiegel}, and
  {Zimmer}]{frankwolfe_nn}
Sebastian {Pokutta}, Christoph {Spiegel}, and Max {Zimmer}.
\newblock {Deep Neural Network Training with Frank-Wolfe}.
\newblock \emph{arXiv e-prints}, art. arXiv:2010.07243, October 2020.

\bibitem[{Ramanujan} et~al.(2019){Ramanujan}, {Wortsman}, {Kembhavi},
  {Farhadi}, and {Rastegari}]{whats_hidden}
Vivek {Ramanujan}, Mitchell {Wortsman}, Aniruddha {Kembhavi}, Ali {Farhadi},
  and Mohammad {Rastegari}.
\newblock {What's Hidden in a Randomly Weighted Neural Network?}
\newblock \emph{arXiv e-prints}, art. arXiv:1911.13299, November 2019.

\bibitem[Sandler et~al.(2018)Sandler, Howard, Zhu, Zhmoginov, and
  Chen]{mobilenetv2}
Mark Sandler, Andrew Howard, Menglong Zhu, Andrey Zhmoginov, and Liang-Chieh
  Chen.
\newblock Mobilenetv2: Inverted residuals and linear bottlenecks.
\newblock In \emph{Proceedings of the IEEE conference on computer vision and
  pattern recognition}, pages 4510--4520, 2018.

\bibitem[Suau et~al.(2020)Suau, Apostoloff, et~al.]{suau2020filter}
Xavier Suau, Nicholas Apostoloff, et~al.
\newblock Filter distillation for network compression.
\newblock In \emph{2020 IEEE Winter Conference on Applications of Computer
  Vision (WACV)}, pages 3129--3138. IEEE, 2020.

\bibitem[{Wang} et~al.(2020){Wang}, {Zhang}, and {Grosse}]{preserve_grad_flow}
Chaoqi {Wang}, Guodong {Zhang}, and Roger {Grosse}.
\newblock {Picking Winning Tickets Before Training by Preserving Gradient
  Flow}.
\newblock \emph{arXiv e-prints}, art. arXiv:2002.07376, February 2020.

\bibitem[Wolfe et~al.(2021)Wolfe, Wang, Kim, and Kyrillidis]{wolfe2021provably}
Cameron~R Wolfe, Qihan Wang, Junhyung~Lyle Kim, and Anastasios Kyrillidis.
\newblock Provably efficient lottery ticket discovery.
\newblock \emph{arXiv preprint arXiv:2108.00259}, 2021.

\bibitem[Ye(2021)]{gfs_github}
Mao Ye.
\newblock Network-pruning-greedy-forward-selection.
\newblock
  \url{https://github.com/lushleaf/Network-Pruning-Greedy-Forward-Selection},
  2021.

\bibitem[Ye et~al.(2020)Ye, Gong, Nie, Zhou, Klivans, and
  Liu]{provable_subnetworks}
Mao Ye, Chengyue Gong, Lizhen Nie, Denny Zhou, Adam Klivans, and Qiang Liu.
\newblock Good subnetworks provably exist: Pruning via greedy forward
  selection.
\newblock In \emph{International Conference on Machine Learning}, pages
  10820--10830. PMLR, 2020.

\bibitem[{Ye} et~al.(2020){Ye}, {Wu}, and {Liu}]{log_pruning}
Mao {Ye}, Lemeng {Wu}, and Qiang {Liu}.
\newblock {Greedy Optimization Provably Wins the Lottery: Logarithmic Number of
  Winning Tickets is Enough}.
\newblock \emph{arXiv e-prints}, art. arXiv:2010.15969, October 2020.

\bibitem[You et~al.(2019)You, Li, Xu, Fu, Wang, Chen, Baraniuk, Wang, and
  Lin]{early_bird}
Haoran You, Chaojian Li, Pengfei Xu, Yonggan Fu, Yue Wang, Xiaohan Chen,
  Richard~G Baraniuk, Zhangyang Wang, and Yingyan Lin.
\newblock Drawing early-bird tickets: Towards more efficient training of deep
  networks.
\newblock \emph{arXiv preprint arXiv:1909.11957}, 2019.

\bibitem[{Yu} et~al.(2017){Yu}, {Li}, {Chen}, {Lai}, {Morariu}, {Han}, {Gao},
  {Lin}, and {Davis}]{nisp}
Ruichi {Yu}, Ang {Li}, Chun-Fu {Chen}, Jui-Hsin {Lai}, Vlad~I. {Morariu},
  Xintong {Han}, Mingfei {Gao}, Ching-Yung {Lin}, and Larry~S. {Davis}.
\newblock {NISP: Pruning Networks using Neuron Importance Score Propagation}.
\newblock \emph{arXiv e-prints}, art. arXiv:1711.05908, November 2017.

\bibitem[Zhang et~al.(2021)Zhang, Wang, Liu, Chen, and Xiong]{zhang2021lottery}
Shuai Zhang, Meng Wang, Sijia Liu, Pin-Yu Chen, and Jinjun Xiong.
\newblock Why lottery ticket wins? a theoretical perspective of sample
  complexity on pruned neural networks.
\newblock \emph{arXiv preprint arXiv:2110.05667}, 2021.

\bibitem[{Zhuang} et~al.(2018){Zhuang}, {Tan}, {Zhuang}, {Liu}, {Guo}, {Wu},
  {Huang}, and {Zhu}]{discrim_prune}
Zhuangwei {Zhuang}, Mingkui {Tan}, Bohan {Zhuang}, Jing {Liu}, Yong {Guo},
  Qingyao {Wu}, Junzhou {Huang}, and Jinhui {Zhu}.
\newblock {Discrimination-aware Channel Pruning for Deep Neural Networks}.
\newblock \emph{arXiv e-prints}, art. arXiv:1810.11809, October 2018.

\end{thebibliography}

\clearpage

\appendix
\section{Experimental Details} \label{A:exp_details}

\subsection{Synthetic Experiments} \label{A:synthetic}
Here, we present the details for generating the synthetic data used for pruning experiments with i-SpaSP on two-layer neural networks presented in Section \ref{S:synthetic_exp}.
It should be noted that the synthetic data generated within this set of experiments does not correspond to the input data matrix $X \in \mathbb{R}^{d_{in} \times B}$ described in Section \ref{S:prelim}.
Rather, the synthetic data that is generated $\{H^{(i)}\}_{i=1}^K$ corresponds to the hidden representation of a two-layer neural network, constructed as $H^{(i)} = \sigma(W^{(0)} \cdot X)$ for some input $X$.
Within these experiments, we generate hidden representations instead of raw input because the row-compressibility ratio $p$ considered within Theorem \ref{T:final_bound} is with respect to the neural network's hidden representations (i.e., not with respect to the input).

Consider the $i$-th synthetic hidden representation matrix $H^{(i)} \in \mathbb{R}^{d_{hid} \times B}$, and assume this matrix has a desired row-compressibility ratio $p$.
For all experiments, we set $B$ (i.e., the size of the dataset) to be 100.
Recall that all entries within $H^{(i)}$ must be non-negative due to the ReLU activation present within \eqref{eq:2layer_nn}.
For each row $j \in [d_{hid}]$ of the matrix, the synthetic hidden representation is generated by $i)$ computing the upper bound on the sum of the values within this $j$-th row as $\frac{R}{i^{\frac{1}{p}}}$ and $ii)$ randomly sampling $B$ non-negative values that, when summed together, do not exceed the upper bound.\footnote{In practice, this is implemented by keeping a running sum and continually sampling numbers randomly within a range that does not exceed the upper bound.}
Here, $R$ is a scalar constant as described in Section \ref{S:prelim}, which we set to $R=1$.
After this process has been completed for each row, the rows of $H^{(i)}$ are randomly shuffled so as to avoid rows being in sorted, magnitude order.

Within the experiments presented in Section \ref{S:synthetic_exp}, we generate three random matrices, following the process described above, for each combination of $d_{hid}$ and $p$.
Then, the average pruning results across each of these three matrices is reported for each experimental setting.
Experiments were also replicated with different settings of $R$, but the results observed were quite similar.
Further, the $W^{(1)}$ matrix used within the synthetic experiments of Section \ref{S:synthetic_exp} was generated using standard Kaiming initialization \citep{2015arXiv150201852H}, which is supported in deep learning packages like PyTorch \citep{paszke2017automatic}.

\subsection{Two-Layer Networks} \label{A:two_layer}
Within this section, we provide all relevant experimental details for pruning experiments with two-layer networks presented in Section \ref{S:two_layer}.

\begin{figure}
    \centering
    \includegraphics[width=0.9\linewidth]{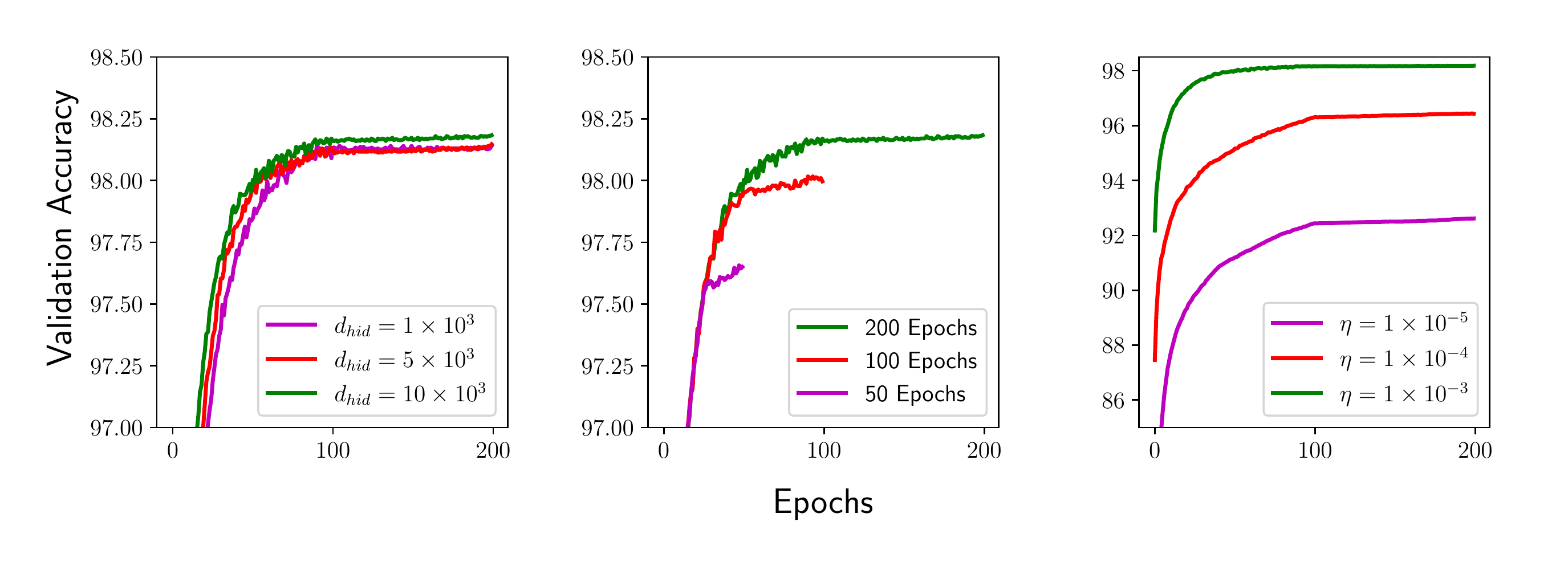}
    \caption{Performance of two-layer networks with different hyperparameter settings on the MNIST validation set. From left to right, subplots depict different hidden dimensions, pre-training epochs, and learning rates. For each of the plots, the best setting with respect to other hyperparameters is displayed.}
    \label{fig:two_layer_hyper}
\end{figure}
\medskip
\noindent \textbf{Baseline Network.}
We begin by describing the baseline two-layer network that was used within pruning experiments on MNIST, as well as relevant details for pre-training the network prior to pruning.
The network used within experiments in Section \ref{A:two_layer} exactly matches the formulation in \eqref{eq:2layer_nn}, and MNIST images are flattened---forming a 784-dimensional input vector---prior to being passed as input to the network.
The network is first pre-trained on the MNIST dataset such that it has fully converged before being used in pruning experiments.
During pre-training, we optimize the network with stochastic gradient descent (SGD) using a batch size of 128 and employ a learning rate step schedule that decays the learning rate $10\times$ after completing 50\% and 75\% of total epochs. 
For all network and pre-training hyperparameters, we identify optimal settings by dividing the MNIST training set randomly (i.e., a random, 80-20 split) into training and validation sets.
Then, performance is measured over the validation set using three separate trials and averaged to identify proper hyperparameter settings.

We find that weight decay and momentum settings do not significantly impact network performance.
Thus, we use no weight decay during pre-training and set momentum to 0.9.
The results of other hyperparameter tuning experiments are provided in Figure \ref{fig:two_layer_hyper}.
We test different hidden dimensions of the two-layer network $d_{hid} \in \{1 \times 10^{3},~5 \times 10^{3},~10 \times 10^{3}\}$, finding that validation performance reaches a plateau when $d_{hid} = 10 \times 10^3$.
We also test numerous possible learning rates $\eta \in \{1 \times 10^{-5},~1 \times 10^{-4},~1 \times 10^{-3}\}$ and pre-training epochs.
Within these experiments, it is determined that a learning rate of $\eta = 1\times 10^{-3}$ with 200 epochs of pre-training acheives the best performance on the validation set; as shown in the right and middle subplots of Figure \ref{fig:two_layer_hyper}.
Thus, the dense networks used for experiments within Section \ref{S:two_layer} has a hidden dimension of $10 \times 10^{3}$ and is pre-trained for 200 epochs using a learning rate of $1 \times 10^{-3}$ prior to pruning being performed.

\medskip
\noindent \textbf{Network Pruning.}
The two-layer network described above is pruned using several greedy selection strategies, including i-SpaSP, GFS, and Top-K.
Each of these pruning strategies selects a subset of neurons within the dense network's hidden layer based on a mini-batch of data from the training set.
Within all experiments, we adopt a batch size of 512 for pruning.

For i-SpaSP, the pruning procedure follows the exact steps outlined in Algorithm \ref{A:cosamp_prune}.
Within each iteration of Algorithm \ref{A:cosamp_prune}, a new mini-batch of data is sampled to compute neuron importance.
We use a fixed number of iterations as the stopping criterion for i-SpaSP.
In particular, we terminate the algorithm and return the pruned model after 20 iterations, as the active set of neurons is consistently stable at this point.

An in-depth description of GFS is provided in \citep{provable_subnetworks}.
For two-layer networks, GFS is implemented by, at each iteration, sampling a new mini-batch of data and finding the neuron that, when included in the (initially empty) active set, yields the largest decrease in loss over the mini-batch.
This process is repeated until the desired size of the pruned network has been reached---i.e., each iteration adds a single neuron to the pruned network and the size of the pruned network is used as a stopping criterion. 
Top-K is a naive, baseline method for greedy selection, which operates by computing hidden representations across the mini-batch and selecting the $k$ neurons with the largest-magnitude hidden activations (i.e., based on summing hidden representation magnitudes across the mini-batch). 
Top-K performs the entire pruning process in one step using a single mini-batch.

\subsection{Deep Convolutional Networks} \label{A:deep_net}
Within this section, we provide relevant details to the large-scale CNN experiments on ImageNet (ILSVRC2012) presented in Section \ref{S:deep_net}.

\medskip
\noindent \textbf{Pruning Deep Networks.}
As described in Section \ref{S:implementation}, i-SpaSP is extended to deep networks by greedily pruning each layer of the network from beginning to end.
We prune filters within each block of the network independently, where a block either consists of two $3 \times 3$ convolution operations separated by batch normalization and ReLU (i.e., \texttt{BasicBlock} for ResNet34) or a depth-wise, separable convolution with a linear bottleneck (i.e., \texttt{InvertedResidualBlock} for MobileNetV2).
We maintain the input and output channel resolution of each block---choosing instead to reduce each block's intermediate channel dimension---by performing pruning as follows:
\begin{itemize}
    \item \texttt{BasicBlock}: We obtain the output of the first convolution (i.e., after batch normalization and ReLU) and perform pruning based on the forward pass of the second convolution, thus eliminating filters from the output of the first convolution and, in turn, the input of the second convolution.
    \item \texttt{InvertedResidualBlock}: We obtain the output of the first $1 \times 1$ point-wise and  $3 \times 3$ depth-wise convolutions (i.e., after normalization and ReLU) and prune the filters of this representation based on the forward pass of the last $1 \times 1$ point-wise linear convolution, thus removing filters from the output of the first point-wise convolution and, in turn, within the depth-wise convolution operation.
\end{itemize}
Thus, each convolutional block is separated into two components.
Input data is passed through the first component to generate a hidden representation, then i-SpaSP performs pruning using forward and backward passes of the second component as in Algorithm \ref{A:cosamp_prune}. 
Intuitively, one can view the two components of each block as the two neural network layers in \eqref{eq:2layer_nn}.
Due to implementation with automatic differentiation, pruning convolutional layers is nearly identical to pruning two-layer network as described in Section \ref{S:implementation}, the only difference being the need to sum over both batch and spatial dimensions in computing importance.

\medskip
\noindent \textbf{Pruning Ratios.}
As mentioned within Section \ref{S:deep_net}, some blocks within each network are more sensitive to pruning than others.
Thus, pruning all layers to a uniform ratio typically does not perform well.
Within ResNet34, four groups of convolutional blocks exist, where each group contains a sequence of convolutional blocks with the same channel dimension.
Following the recommendations of \cite{conv_filt_prune}, we avoid pruning the fourth group of convolutional blocks, any strided blocks, or the last block in any group.
The 2.69 GFlop model in Table \ref{tab:imagenet} uses a uniform ratio of 40\% for all three groups (i.e., 60\% of filters are eliminated from each block). 
For the 2.13 GFlop model, the first group is pruned to a ratio of 20\%, while the second and third groups are pruned to a ratio of 30\%. 

Within MobileNetV2, we simply avoid pruning network blocks that are either strided or have different input and output channel dimensions, leaving the following blocks to be pruned: 3, 5, 6, 8, 9, 10, 12, 13, 15, 16.
We find that blocks 3, 5, 8, 15, and 16 are especially sensitive to pruning (i.e., pruning these layers causes noticeable performance degradation), so we avoid aggressive pruning upon these blocks.
The 260 MFlop model is generated by pruning only blocks 6, 9, 10, 12, and 13 to a ratio of 40\%.
For the 242 MFlop and 220 MFlop models, we prune sensitive (i.e., blocks 3, 4, 8, 15, and 16) and not sensitive (i.e., blocks 6, 9, 10, 12, and 13) blocks with ratios of 40\%/80\% and 30\%/70\%, respectively.
For the 220 MFlop model, we also prune blocks 2, 11, and 17 to a ratio of 90\% to further decrease the FLOP count. 

\medskip 
\noindent \textbf{Pruning Hyperparameters.}
Pruning of each layer is performed with a batch of 1280 data examples for both ResNet34 and MobileNetV2 architectures.
This batch of data can be separated into multiple mini-batches (e.g., of size 256) during pruning.
Such an approach allows the computation of i-SpaSP to be performed on a GPU (i.e., all of the data examaples cannot simultaneously fit within memory of a typical GPU), but this modification does not change the runtime of behavior of i-SpaSP.
Furthermore, 20 total i-SpaSP iterations are performed in pruning each layer.
We find that adding more data or pruning iterations beyond these settings does not provide any benefit to i-SpaSP performance.

\medskip
\noindent \textbf{Fine-tuning Details.}
For fine-tuning, we perform 90 epochs of fine-tuning using SGD with momentum and employ a cosine learning rate decay schedule over the entire fine-tuning process beginning with a learning rate of 0.01.
For fine-tuning between the pruning of layers (i.e., a small amount of fine-tuning is performed after each layer is pruned), we simply adopt the same fine-tuning setup with a fixed learning rate of 0.01 and fine-tuning continues for one epoch.
Our learning rate setup for fine-tuning matches that of \citep{provable_subnetworks, log_pruning}, but we choose to perform fewer epochs of fine-tuning, as we find network performance does not improve much after 90 epochs. 
In all cases, we adopt standard augmentation procedures for the ImageNet dataset during fine-tuning (i.e., normalizing the data and performing random crops and flips). 

\subsection{Transformer Networks} \label{A:transformers}
Within this section, we provide all relevant experimental details for pruning experiments with two-layer networks presenting in Section \ref{S:transformers}.

\medskip
\noindent \textbf{Baseline Network.} We begin by describing the baseline network used within the structured attention head pruning experiments.
The network architecture used in all experiments is the multi-lingual BERT (mBERT) model \citep{devlin2018bert}.
More specifically, we adopt the BERT-base-cased variant of the architecture, which is pre-trained on a corpus of over 100 different languages.
Beginning with this pre-trained network architecture, we fine-tune the BERT model on the XNLI dataset for two epochs with the AdamW optimizer and a learning rate of $5 \times 10^{-5}$ that is decayed linearly throughout the training process. 
After the fine-tuning process, this model achieves an accuracy of 71.02\% on the XNLI test set. 

\medskip
\noindent \textbf{Network Pruning.} For all pruning experiments, we adopt a structured pruning approach that selectively removes entire attention heads from network layers. 
We adopt two different pruning ratios: 25\% and 40\%.\footnote{Here, the pruning ratio refers to the number of attention heads remaining after pruning. For example, a 12-head layer pruned at a ratio of 25\% would have three remaining heads after pruning completes.}
These ratios correspond to three and five attention heads (i.e., each layer of mBERT originally has 12 attention heads) being retained within each layer of the network.

We compare the performance of i-SpaSP to a uniform pruning approach (i.e., randomly select attention heads to prune within each layer) and a sensitivity-based, global pruning approach that was previously proposed for structured pruning of attention heads \citep{michel2019sixteen}.
For i-SpaSP and uniform pruning, one layer is pruned at a time, starting from the first layer of the network and ending at the final layer.
For the sensitivity-based pruning technique, the pruning process is global and all layers are pruned simultaneously.
After pruning has been completed, networks are fine-tuned for 0.5 epochs using the AdamW optimizer with a learning rate of $5\times 10^{-5}$ that is decayed linearly throughout the fine-tuning process.
For i-SpaSP and the sensitivity-based pruning approach, we adopt the same setting as previous experiments (i.e., see Section \ref{S:deep_net}) and utilize a subset of 1280 training data examples for pruning.
Additionally, i-SpaSP performs 20 total iterations for the pruning of each layer, which again reflects the settings used within previous experiments. 

\section{Main Proofs} \label{A:proofs}

\subsection{Properties of $\mu(\cdot)$}
Prior to providing any proofs of our theoretical results, we show a few properties of the function $\mu(\cdot)$ defined over matrices that will become useful in deriving later results.
Given $A \in \mathbb{R}^{m \times n}$, $\mu (A): \mathbb{R}^{m \times n} \longrightarrow \mathbb{R}^m$ and is defined as follows:

\begin{align*}
    \mu(A) = \sum\limits_{i=1}^{n} A_{:, i} 
\end{align*}

In words, $\mu(A)$ sums all columns within an arbitraty matrix $A$ to produce a single column vector.
We now provide a few useful properties of $\mu(\cdot)$.
\begin{lemma} \label{L:mu}
    Consider two arbitrary matrices $A \in \mathbb{R}^{m \times n}$ and $B \in \mathbb{R}^{n \times p}$. The following properties of $\mu(\cdot)$ must be true.
    \begin{align*}
        \mu(A\cdot B) &= A \cdot \mu(B)\\
        \mu(A + B) &= \mu(A) + \mu(B)
    \end{align*}
\end{lemma}
\begin{proof}
Given the matrices $A \in \mathbb{R}^{m \times n}$ and $B \in \mathbb{R}^{n \times p}$, we can arrive at the first property as follows:
\begin{align*}
    \mu(A\cdot B) &= \sum\limits_{i=1}^p (A\cdot B)_{:, i}
    = \sum\limits_{i=1}^p A \cdot B_{:, i} 
    = A \cdot \left ( \sum\limits_{i=1}^p B_{:, i} \right)
    = A\cdot \mu(B)
\end{align*}
\noindent Furthermore, the second property can be shown as follows:
\begin{align*}
    \mu(A + B) &= \sum\limits_{i=1}^p (A + B)_{:, i}
    = \sum\limits_{i=1}^p A_{:, i} + B_{:, i}
    = \sum\limits_{i=1}^p A_{:, i} + \sum\limits_{j=1}^p B_{:, j}
    = \mu(A) + \mu(B)
\end{align*}

\end{proof}

\subsection{Proof of Lemma \ref{L:hidden_res}} \label{A:hidden_res_proof}
Here, we provide a proof for Lemma \ref{L:hidden_res} from Section \ref{S:theory}.
\medskip
\begin{proof}
Consider an arbitrary iteration of Algorithm \ref{A:cosamp_prune}.
We define the active neurons within the pruned network entering this iteration as $\mathcal{S}$.
The notation $\mathcal{S}'$ will be used to refer to the active neurons in the pruned model after the completion of the iteration.
Similarly, we will use the $\cdot '$ notation to refer to other constructions after the iteration completes (e.g., $V$ becomes $V'$, $R$ becomes $R'$, etc.)

We begin by defining and expanding the definitions of the matrices $U$, $V$, $R$, and $Y$ considering the substitution of $H = Z + E$;
\begin{align*}
    U &= W^{(1)} H = W^{(1)} (Z + E), \\
    V &= U - W^{(1)}H_{\mathcal{S}, :} = W^{(1)} (Z + E) - W^{(1)}(Z + E)_{\mathcal{S}, :}, \\
    R &= Z - Z_{\mathcal{S}, :}, \\
    Y &= (W^{(1)})^T V.
\end{align*}
\noindent Now, observe the following about $Y$.
\begin{align*}
    Y &= (W^{(1)})^\top \cdot V\\
    &= (W^{(1)})^\top \left ( W^{(1)}(Z + E) - W^{(1)}(Z + E
    )_{\mathcal{S}, :} \right ) \\
    &= (W^{(1)})^\top \left ( W^{(1)} Z + W^{(1)} E - W^{(1)} (Z_{\mathcal{S}, :}) - W^{(1)} (E_{\mathcal{S}, :}) \right) \\
    &= (W^{(1)})^\top \left (W^{(1)} (Z - Z_{\mathcal{S}, :}) + W^{(1)} (E - E_{\mathcal{S}, :}) \right) \\
    &= (W^{(1)})^\top \left (W^{(1)}(Z - Z_{\mathcal{S}, :}) + W^{(1)} \hat{E} \right) \\
    &= (W^{(1)})^\top W^{(1)}(Z - Z_{\mathcal{S}, :}) + (W^{(1)})^\top W^{(1)} \hat{E},
\end{align*}
where we denote $\hat{E} = E - E_{\mathcal{S}, :}$.
It should be noted that $\hat{E}$ is simply the matrix $E$ with all the $i$-th rows set to zero where $i \in \mathcal{S}$. 
Invoking Lemma \ref{L:mu} in combination with the above expression, we can show the following.
\begin{align*}
    \mu(Y) &= (W^{(1)})^\top W^{(1)} \mu(Z - Z_{\mathcal{S}, :}) + (W^{(1)})^\top W^{(1)} \mu(\hat{E}) \\
    &= (W^{(1)})^\top W^{(1)} \left (\mu(R) + \mu(\hat{E}) \right).
\end{align*}

Based on the definition of $H = Z + E$, $R$ must be $s$-row-sparse.
This is because $Z$ is $s$-row-sparse and $R = Z - Z_{\mathcal{S}, :}$.
Thus, the number of non-zero rows within $R$ must be less than or equal to $s$ (i.e., the rows that are not subtracted out of $Z$ by $Z_{\mathcal{S}, :}$).
Denote $\Delta = \rsupp(R)$, where $|\Delta| \leq s$.
From the definition of $\Omega$ in Algorithm \ref{A:cosamp_prune}, we have the following properties of $Y$:
\begin{align*}
    ~~~~~~\|\mu(Y)|_{\Delta}\|_2 \leq \|\mu(Y)|_{\Omega}\|_2 \quad
    \stackrel{i)}{\Longrightarrow} \quad \|\mu(Y)|_{\Delta \textbackslash \Omega}\|_2 \leq \|\mu(Y)|_{\Omega \textbackslash \Delta}\|_2
\end{align*}
where $i)$ follows by removing indices within $\Omega \cap \Delta$.
From here, we can upper bound the value of $\|\mu(Y)|_{\Omega \textbackslash \Delta}\|_2$ as follows.
\begin{align*}
    \|\mu(Y)|_{\Omega \textbackslash \Delta}\|_2 &= \left\| \left((W^{(1)})^\top W^{(1)}\mu(R) \right)|_{\Omega \textbackslash \Delta} + \left((W^{(1)})^\top W^{(1)} \mu(\hat{E}) \right)|_{\Omega \textbackslash \Delta} \right\|_2 \\
    &= \left\| (W^{(1)}_{:, \Omega \textbackslash \Delta})^\top W^{(1)}\mu(R) + (W^{(1)}_{:, \Omega \textbackslash \Delta})^\top W^{(1)} \mu(\hat{E}) \right\|_2 \\
    &\stackrel{i)}{\leq} \left\| (W^{(1)}_{:, \Omega \textbackslash \Delta})^\top W^{(1)}\mu(R) \right\|_2 + \left \| (W^{(1)}_{:, \Omega \textbackslash \Delta})^\top W^{(1)} \mu(\hat{E}) \right\|_2 \\
    &\stackrel{ii)}{\leq} \delta_{3s} \|\mu(R)\|_2  + \left \| (W^{(1)}_{:, \Omega \textbackslash \Delta})^\top W^{(1)} \mu(\hat{E}) \right\|_2 \\
    &\stackrel{iii)}{\leq} \delta_{3s} \|\mu(R)\|_2 + \sqrt{1 + \delta_{2s}} \|W^{(1)} \mu(\hat{E})\|_2 \\
    &\stackrel{iv)}{\leq} \delta_{3s} \|\mu(R)\|_2 + (1 + \delta_{2s}) \|\mu(\hat{E})\|_2 + \frac{1 + \delta_{2s}}{\sqrt{2s}}\|\mu(\hat{E})\|_1
\end{align*}
where $i)$ follows from the triangle inequality, $ii)$ holds from Corrolary 3.3 on RIP properties of $W^{(1)}$ in \citep{cosamp} because $|\Omega \cup \Delta| \leq 3s$,
$iii)$ holds from Proposition 3.1 on RIP properties of $W^{(1)}$ in \citep{cosamp} because $|\Omega \textbackslash \Delta| \leq 2s$, and $iv)$ holds from Proposition 3.5 on RIP properties of $W^{(1)}$ in \citep{cosamp}.
Now, we can also lower bound the value of $\|\mu(Y)|_{\Delta \textbackslash \Omega}\|_2$ as follows:
\begin{align*}
    \|\mu(Y)|_{\Delta \textbackslash \Omega}\|_2 &= \left \| \left((W^{(1)})^\top W^{(1)}\mu(R)\right)|_{\Delta \textbackslash \Omega} + \left ((W^{(1)})^\top W^{(1)} \mu(\hat{E}) \right )|_{\Delta \textbackslash \Omega}  \right \|_2 \\ 
    &= \left \| (W^{(1)}_{:, \Delta \textbackslash \Omega})^\top W^{(1)} \mu(R) + (W^{(1)}_{:, \Delta \textbackslash \Omega})^\top W^{(1)} \mu(\hat{E}) \right \|_2 \\
    &= \left \| (W^{(1)}_{:, \Delta \textbackslash \Omega})^\top W^{(1)} \mu(R)|_{\Delta \textbackslash \Omega} + (W^{(1)}_{:, \Delta \textbackslash \Omega})^\top W^{(1)} \mu(R)|_{\Omega} + (W^{(1)}_{:, \Delta \textbackslash \Omega})^\top W^{(1)} \mu(\hat{E}) \right \|_2 \\
    &\stackrel{i)}{\geq} \left \| (W^{(1)}_{:, \Delta \textbackslash \Omega})^\top W^{(1)} \mu(R)|_{\Delta \textbackslash \Omega}\right\|_2 - \left\|(W^{(1)}_{:, \Delta \textbackslash \Omega})^\top W^{(1)} \mu(R)|_{\Omega}\right \|_2 - \left\| (W^{(1)}_{:, \Delta \textbackslash \Omega})^\top W^{(1)} \mu(\hat{E}) \right \|_2 \\
    &\stackrel{ii)}{\geq} (1 - \delta_s)\|\mu(R)|_{\Delta \textbackslash \Omega}\|_2 - \left\|(W^{(1)}_{:, \Delta \textbackslash \Omega})^\top W^{(1)} \mu(R)|_{\Omega}\right \|_2 - \left\| (W^{(1)}_{:, \Delta \textbackslash \Omega})^\top W^{(1)} \mu(\hat{E}) \right \|_2 \\
    &\stackrel{iii}{\geq} (1 - \delta_s)\|\mu(R)|_{\Delta \textbackslash \Omega}\|_2 - \delta_s \|\mu(R)|_{\Omega}\|_2 - \sqrt{1 + \delta_s}\|W^{(1)}\mu(\hat{E})\|_2\\
    &\stackrel{iv)}{\geq}(1 - \delta_s)\|\mu(R)|_{\Delta \textbackslash \Omega}\|_2 - \delta_s \|\mu(R)\|_2 - (1 + \delta_{2s})\|\mu(\hat{E})\|_2 + \frac{1 + \delta_{2s}}{\sqrt{2s}}\|\mu(\hat{E})\|_1 
\end{align*}
where $i)$ follows from the triangle inequality.
$ii)$ follows from Proposition 3.1 in \citep{cosamp} and $iii)$ follows from Corollary 3.3 and Proposition 3.1 in \citep{cosamp} because $|\Delta \textbackslash \Omega| \leq s$.
Finally, $iv)$ follows from Proposition 3.5 within \citep{cosamp}.
By noting that $\mu(R)|_{\Delta \textbackslash \Omega} = \mu(R)|_{\Omega^c}$ and combining the lower and upper bounds above, we derive the following:
\begin{align} \label{eq:r_upperlower_no_sparse}
    \|\mu(R)|_{\Omega^c}\|_2 \leq \frac{(\delta_s + \delta_{3s})\|\mu(R)\|_2 + 2\left((1 + \delta_{2s})\|\mu(\hat{E})\|_2 + \frac{1 + \delta_{2s}}{\sqrt{2s}}\|\mu(\hat{E})\|_1 \right) }{1 - \delta_s}
\end{align}

\noindent From here, we can show the following:
\begin{equation} \label{eq:romegac_no_sparse}
\begin{split}
    \|\mu(R)|_{\Omega^c}\|_2 &\stackrel{i)}{\geq} \|\mu(R)|_{\Omega^{\star c}}\|_2\\
    &= \|\mu(Z - Z_{\mathcal{S}, :})|_{\Omega^{\star c}}\|_2 \\
    &= \|(\mu(Z) - \mu(Z)|_{\mathcal{S}})|_{\Omega^{\star c}} \|_2 \\
    &\stackrel{ii)}{=} \|\mu(Z)|_{\Omega^{\star c}}\|_2 \\
    &= \|\mu(H - E)|_{\Omega^{\star c}}\|_2 \\
    &=\| \mu(H)|_{\Omega^{\star c}} - \mu(E)|_{\Omega^{\star c}} \|_2 \\
    &\stackrel{iii)}{\geq} \|\mu(H)|_{\Omega^{\star c}}\|_2 - \|\mu(E)|_{\Omega^{\star c}}\|_2
\end{split}
\end{equation}
where $i)$ follows from the fact that $\Omega \subseteq \Omega^\star$, $ii)$ follows from the fact that $\mathcal{S} \subset \Omega^\star$, and $iii)$ follows from the triangle inequality.
Then, as a final step, we make the following observation, where we leverage notation from Algorithm \ref{A:cosamp_prune}:
\begin{equation} \label{eq:prune_full_ds}
\begin{split}
    \|\mu(H) - \mu(H)|_{\mathcal{S}'}\|_2 &= \| \mu(H) - b_s \|_2 \\
    &= \| (\mu(H) - b) + (b - b_s) \|_2 \\
    &\stackrel{i)}{\leq} \|\mu(H) - b\|_2 + \|b - b_s\|_2 \\
    &\stackrel{ii)}{\leq} 2\|\mu(H) - b\|_2 \\
    &= 2\|\mu(H) - \mu(H)|_{\Omega^\star}\|_2
\end{split}
\end{equation}
where $i)$ follow from the triangle inequality and $ii)$ follows from the fact that $b_s$ is the best $s$-sparse approximation to $b$ (i.e., $\mu(H)$ is a worse $s$-sparse approximation with respect to the $\ell_2$ norm). 
Now, noticing that $\|\mu(H) - \mu(H)|_{\Omega^{\star}}\|_2 = \|\mu(H)|_{\Omega^{\star c}}\|_2$, we can aggregate all of this information as follows:
\begin{align*}
    \|\mu(H) - \mu(H)|_{\mathcal{S}'}\|_2 - 2\|\mu(\hat{E})\|_2 &\stackrel{i)}{\leq} \|\mu(H) - \mu(H)|_{\mathcal{S}'}\|_2 - 2\|\mu(E)|_{\Omega^{\star c}}\|_2 \\
    &\stackrel{ii)}{\leq} 2 \left(\|\mu(H) - \mu(H)|_{\Omega^{\star}}\|_2 - \|\mu(E)|_{\Omega^{\star c}}\|_2 \right) \\
    &= 2 \left (\|\mu(H)|_{\Omega^{\star c}}\|_2 - \|\mu(E)|_{\Omega^{\star c}}\|_2 \right) \\
    &\stackrel{iii)}{\leq} 2\|\mu(R)|_{\Omega^c}\|_2\\
    &\stackrel{iv)}{\leq} \frac{2(\delta_s + \delta_{3s})\|\mu(R)\|_2 + 4\left((1 + \delta_{2s})\|\mu(\hat{E})\|_2 + \frac{1 + \delta_{2s}}{\sqrt{2s}}\|\mu(\hat{E})\|_1 \right) }{1 - \delta_s}
\end{align*}
where $i)$ holds because  $\mathcal{S} \subset \Omega^\star$, $ii)$ holds from \eqref{eq:prune_full_ds}, $iii)$ holds from \eqref{eq:romegac_no_sparse}, and $iv)$ hold from \eqref{eq:r_upperlower_no_sparse}.
By simply moving the $\|\mu(\hat{E})\|_2$ term to the other side of the inequality, we get the following:
\begin{align*}
    \|\mu(H) - \mu(H)|_{\mathcal{S}'}\|_2 \leq \frac{2(\delta_s + \delta_{3s})\|\mu(R)\|_2 + (6 + 4\delta_{2s} - 2\delta_s)\|\mu(\hat{E})\|_2 + \frac{4(1 + \delta_{2s})}{\sqrt{2s}}\|\mu(\hat{E})\|_1 }{1 - \delta_s}
\end{align*}
\noindent From here, we recover $\mu(H) - \mu(H)|_{\mathcal{S}}$ from $\mu(R)$ as follows:
\begin{align*}
    \|\mu(R)\|_2 &= \|\mu(Z - Z|_{\mathcal{S}})\|_2\\
    &= \|\mu(H - E - (H - E)|_{\mathcal{S}})\|_2 \\
    &= \|\mu(H - H|_{\mathcal{S}}) - \mu(E - E|_{\mathcal{S}})\|_2 \\
    &\stackrel{i)}{\leq} \|\mu(H - H|_{\mathcal{S}})\|_2 + \|\mu(E - E|_{\mathcal{S}})\|_2 \\
    &\leq \|\mu(H) - \mu(H)|_{\mathcal{S}}\|_2 + \|\mu(\hat{E})\|_2
\end{align*}
where $i)$ holds from the triangle inequality.
Combining this with the final inequality derived above, we arrive at the following recursion for error between pruned and dense model hidden layers over iterations of Algorithm \ref{A:cosamp_prune}:
\begin{align*}
        \|\mu(H) - \mu(H)|_{\mathcal{S}'}\|_2 \leq \tfrac{2(\delta_s + \delta_{3s})\|\mu(H) - \mu(H)|_{\mathcal{S}}\|_2 + (6 + 2\delta_{3s} + 4\delta_{2s})\|\mu(\hat{E})\|_2 + \frac{4(1 + \delta_{2s})}{\sqrt{2s}}\|\mu(\hat{E})\|_1 }{1 - \delta_s}
\end{align*}

We now adopt an identical numerical assumption as \citep{cosamp} and assume that $W^{(1)}$ has restricted isometry constant $\delta_{4s} \leq 0.1$.
Then, noting that $\delta_s, \delta_{2s}, \delta_{3s} \leq \delta_{4s} \leq 0.1$ per Corollary 3.4 in \citep{cosamp}, we substitute the restricted isometry constants into the above recursion to yield the following expression:
\begin{align*}
        \|\mu(H) - \mu(H)|_{\mathcal{S}'}\|_2 &\leq 0.444\|\mu(H) - \mu(H)|_{\mathcal{S}}\|_2 + 7.333\|\mu(\hat{E})\|_2 + \frac{4.888}{\sqrt{2s}}\|\mu(\hat{E})\|_1\\
        &= 0.444\|\mu(H) - \mu(H)|_{\mathcal{S}}\|_2 + 7.333\|\mu(\hat{E})\|_2 + \frac{3.456}{\sqrt{s}}\|\mu(\hat{E})\|_1
\end{align*}
If we invoke Lemma \ref{L:mu}, unroll the above recursion over $t$ iterations of Algorithm \ref{A:cosamp_prune}, notice it is always true that $\|\hat{E}\|_2 \leq \|E\|_2$ and $\|\hat{E}\|_1 \leq \|E\|_1$, and draw upon the property of $\ell_1$ and $\ell_2$ vector norms that $\|\cdot\|_2 \leq \|\cdot\|_1$ we arrive at the desired result:
\begin{align*}
    \|\mu(H - H_{\mathcal{S}_t, :})\|_2 &= \|\mu(H) - \mu(H)|_{\mathcal{S}_t}\|_2 \\
    &\leq (0.444)^t\|\mu(H) - \mu(H)|_{\mathcal{S}_0}\|_2 + 14 \|\mu(E)\|_2 + \frac{7}{\sqrt{s}}\|\mu(E)\|_1 \\
    &\leq (0.444)^t\|\mu(H)\|_2 + \left(14 + \frac{7}{\sqrt{s}}\right)\|\mu(E)\|_1
\end{align*}


\end{proof}

\subsection{Proof of Theorem \ref{T:vfrob}}
Here, we provide a proof of Theorem \ref{T:vfrob} from Section \ref{S:theory}.

\medskip
\begin{proof}
Consider an arbitrary iteration $t$ of Algorithm \ref{A:cosamp_prune} with a set of active neurons within the pruned model denoted by $\mathcal{S}_t$.
The matrix $V_t$ contains the residual between the dense and pruned model output over the entire dataset, as formalized below:
\begin{align*}
    V_t &= U - (W^{(1)} \cdot H_{\mathcal{S}_t, :})\\
    &= (W^{(1)}\cdot H) - (W^{(1)} \cdot H_{\mathcal{S}_t, :}) \\
    &= W^{(1)} \cdot (H - H_{\mathcal{S}_t, :})
\end{align*}

\noindent Consider the squared Frobenius norm of the matrix $V_t$.
We can show the following:
\begin{align*}
    \|V_t\|_F^2 &= \|W^{(1)} \cdot (H - H_{\mathcal{S}_t, :}) \|_F^2 \\
    &= \sum\limits_{i=1}^{d_{out}} \sum\limits_{j=1}^B \left (\sum\limits_{z=1}^{d_{hid}} W^{(1)}_{iz} (H - H_{\mathcal{S}_t, :})_{zj} \right)^2 \\
    &= \sum\limits_{i=1}^{d_{out}} \sum\limits_{j=1}^B \left (W^{(1)}_{i, :}\cdot (H - H_{\mathcal{S}_t, :})_{:, j} \right)^2 \\
    &= \sum\limits_{i=1}^{d_{out}} \sum\limits_{j=1}^B \left (|W^{(1)}_{i, :}\cdot (H - H_{\mathcal{S}_t, :})_{:, j}| \right)^2 \\
    &\stackrel{i)}{\leq} \sum\limits_{i=1}^{d_{out}} \sum\limits_{j=1}^B \left ( \|W^{(1)}_{i, :}\|_2 \cdot \|(H - H_{\mathcal{S}_t, :})_{:, j}\|_2 \right)^2 \\
    &= \sum\limits_{i=1}^{d_{out}} \sum\limits_{j=1}^B \|W^{(1)}_{i, :}\|_2^2 \cdot \|(H - H_{\mathcal{S}_t, :})_{:, j}\|_2^2 \\
    &= \sum\limits_{i=1}^{d_{out}} \sum\limits_{j=1}^B \left ( \sum\limits_{k=1}^{d_{hid}} (W^{(1)}_{ik})^2 \right) \left( \sum\limits_{z=1}^{d_{hid}} (H - H_{\mathcal{S}_t, :})^2_{zj} \right) \\
    &= \left ( \sum\limits_{i=1}^{d_{out}} \sum\limits_{k=1}^{d_{hid}} (W^{(1)}_{ik})^2 \right ) \left ( \sum\limits_{j=1}^B \sum\limits_{z=1}^{d_{hid}} (H - H_{\mathcal{S}_t, :})_{zj}^2 \right) \\
    &=  \|W^{(1)}\|_F^2 \left (\sum\limits_{z=1}^{d_{hid}} \sum\limits_{j=1}^B (H - H_{\mathcal{S}_t}, :)_{zj}^2 \right)\\
    &\stackrel{ii)}{\leq} \|W^{(1)}\|_F^2 \left (\sum\limits_{z=1}^{d_{hid}} \mu(H - H_{\mathcal{S}_t, :})_z^2 \right)\\
    &= \|W^{(1)}\|_F^2 \cdot \|\mu(H - H_{\mathcal{S}_t, :})\|_2^2
\end{align*}
where $i)$ follows from the Cauchy-Schwarz inequality and $ii)$ follows from the fact that all entries of $H - H_{\mathcal{S}_t, :}$ are non-negative. 
From here, we simply invoke Lemma \ref{L:hidden_res} to arrive at the desired result.

\begin{align*}
    \|V_t\|_F &\leq \|W^{(1)}\|_F \cdot \|\mu(H - H_{\mathcal{S}_t, :})\|_2  \\&\leq \|W^{(1)}\|_F \cdot \left ( (0.444)^t \|\mu(H)\|_2 + \left ( 14 + \frac{7}{\sqrt{s}}\right)\|\mu(E)\|_1 \right)
\end{align*}

\end{proof}

\subsection{Proof of Lemma \ref{L:noise_bound}}
Here, we provide the proof for Lemma \ref{L:noise_bound} from Section \ref{S:theory}.

\medskip
\begin{proof}
Assume that we choose the $s$-row-sparse $Z$ matrix within $H = Z + E$ as follows
\begin{align*}
    Z = H_{\mathcal{D}, :},~\text{where}~\mathcal{D} = \argmin_{\mathcal{D} \subset [d_{hid}]} \|\mu(H) - \mu(H)|_{\mathcal{D}}\| 
\end{align*}
The norm used within the above expression is arbitrary, as we simply care to select indices $\mathcal{D}$ such that $\mu(H)|_{\mathcal{D}}$ contains the largest-magnitude components of $\mu(H)$.
Such a property will be true for any $\ell_p$ norm selection within the above equation, as they all ensure $\mu(H)|_{\mathcal{D}}$ is the best possible $s$-sparse approximation of $\mu(H)$.
From here, we have that $E = H - H_{\mathcal{D}}$, which yields the following:
\begin{align*}
    \|\mu(E)\|_1 &= \|\mu(H - H_{\mathcal{D}})\|_1 \\
    &\stackrel{i)}{=} \|\mu(H) - \mu(H)|_{\mathcal{D}}\|_1 \\
    &\stackrel{ii)}{\leq} R \cdot \frac{s^{1 - \frac{1}{p}}}{\frac{1}{p} - 1}
\end{align*}
where $i)$ holds from Lemma \ref{L:mu} and $ii)$ is a bound on the $\ell_1$ norm of the best-possible $s$-sparse approximation of a $p$-compressible vector; see Section 2.6 in \citep{cosamp}.

\end{proof}

\subsection{Proof of Theorem \ref{T:multi_layer}} \label{A:multi_layer}
Here, we generalize our theoretical results beyond single-hidden-layer networks to arbitrary depth networks.
For this section, we define the forward pass in a $L$-hidden-layer network via the following recursion:
\begin{align*}
    H^{(\ell)} = \sigma (W^{(\ell - 1)} \cdot H^{(\ell - 1)})
\end{align*}
where $\sigma$ denotes the nonlinear ReLU activation applied at each layer and we have weight matrices $\mathcal{W} = \{W^{(0)}, W^{(1)}, \dots, W^{(L)} \}$ and hidden representations $\mathcal{H} = \{H^{(0)}, H^{(1)}, \dots, H^{(L)} \} $ for each of the $L$ layers within the network.
Here, $H^{(L)}$ denotes the network output and $H^{(0)} = \sigma(W^{(0)} \cdot X)$.
It should be noted that no ReLU activation is applied at the final network output (i.e., $H^{(L)} = W^{(L)} \cdot H^{(L - 1)}$).
We denote the pre-activation values each hidden representation as $U^{(\ell)} = W^{(\ell)} \cdot H^{(\ell - 1)}$.

\medskip
\begin{proof}
We denote the active set of neurons after $t$ iterations at each layer $\ell$ as $\mathcal{S}^{(\ell)}_t$.
The dense network output is given by $U^{(L)} = H^{(L)}$, and we compute the residual between pruned and dense networks at layer $\ell$ as $V^{(\ell)}_t = U^{(\ell)} - W^{(\ell)}_{:, \mathcal{S}^{(\ell)}_t} \cdot H^{(\ell - 1)}_{\mathcal{S}^{(\ell)}_t, :}$.
We further denote $U_{t}^{(\ell)'} = W^{(\ell)}_{:, \mathcal{S}^{(\ell)}_t} \cdot H^{(\ell - 1)}_{\mathcal{S}^{(\ell)}_t, :}$ as shorthand for intermediate, pre-activation outputs of the pruned network.
For simplicity, we assume all hidden layers of the multi-layer network have the same number of hidden neurons $d$ and are pruned to a size $s$, such that $s \ll d$.
Consider the final output representation of an $L$-hidden-layer network $H^{(L)}$.
We will now bound $\|V^{(L)}_t\|_F$, revealing that a recursion can be derived over the layers of the network to upper bound the norm of the final layer's residual between pruned and dense networks.

Following previous proofs, we decompose the hidden representation at each layer prior to the output layer as $H^{(\ell)} = Z^{(\ell)} + E^{(\ell)} + \Delta^{(\ell)}$, where $Z^{(\ell)}$ is an $s$-row-sparse matrix and $E^{(\ell)}$ and $\Delta^{(\ell)}$ are both arbitrary-valued matrices.
$E^{(\ell)}$ denotes the same arbitrarily-valued matrix from the previous $H = Z + E$ formulation and $\Delta^{(\ell)}$ captures all error introduced by  pruning layers prior to $\ell$ within the network.
$E^{(\ell)}$ and $\Delta^{(L)}$ can also be combined into a single matrix as $\Xi^{(\ell)} = E^{(\ell)} + \Delta^{(\ell)}$.
Now, we consider the value of $\|V^{(L)}_t\|_F$:
\begin{align}
    \|V^{(L)}_t\|_F &\stackrel{i}{=} \|W^{(L)}\|_F \cdot \left( (0.444)^t \|\mu(H^{(L - 1)})\|_2 + \left(14 + \frac{7}{\sqrt{s}}\right)\|\mu(\Xi^{(L - 1)})\|_1 \right) \nonumber\\
    &= \|W^{(L)}\|_F \cdot \left( (0.444)^t \|\mu(H^{(L - 1)})\|_2 + \left(14 + \frac{7}{\sqrt{s}}\right)\|\mu(E^{(L - 1)}) + \mu(\Delta^{(L - 1)})\|_1 \right) \nonumber \\
    &\stackrel{ii}{\leq}\|W^{(L)}\|_F \cdot \left( (0.444)^t \|\mu(H^{(L - 1)})\|_2 + \left(14 + \frac{7}{\sqrt{s}}\right) \left (\|\mu(E^{(L - 1)})\|_1 + \|\mu(\Delta^{(L - 1)})\|_1 \right ) \right) \nonumber \\
    &\stackrel{iii}{\approx} \|W^{(L)}\|_F \cdot \left(14 + \frac{7}{\sqrt{s}}\right) \left (\|\mu(E^{(L - 1)})\|_1 + \|\mu(\Delta^{(L - 1)})\|_1 \right ) \label{eq:multi_last_layer}
\end{align}
where $i$ holds from Theorem \ref{T:vfrob}, $ii$ holds from the triangle inequality, and $iii$ holds by eliminating all terms that approach zero with enough pruning iterations $t$.
Now, notice that $\|\mu(\Delta^{(L - 1)})\|_1$ characterizes the error induced by pruning within the previous layer of the network, following the ReLU activation.
Thus, our expression for the error induced by pruning on network output is now expressed with respect to the pruning error of the previous layer, revealing that a recursion can be derived for pruning error over the entire multi-layer network.
We now expand the expression for $\|\mu(\Delta^{(\ell)})\|_1$ (i.e., we use arbitrary layer $\ell$ to enable a recursion over layers to be derived).

\begin{align*}
    \|\mu(\Delta^{(\ell)})\|_1 &\stackrel{i)}{=} \| \mu( \sigma(U^{(\ell)}) - \sigma(U_t^{(\ell)'}))\|_1 \\
    &\stackrel{ii)}{\leq} \|\mu(\sigma(U^{(\ell)} - U_t^{(\ell)'}))\|_1 \\
    &= \|\mu(\sigma(V^{(\ell)}_t))\|_1 \\
    &= \left \|\mu \left (\sigma \left (W^{(\ell)} \cdot \left (H^{(\ell - 1)} - H^{(\ell - 1)}_{\mathcal{S}^{(\ell - 1)}_t, :} \right ) \right ) \right) \right \|_1 \\
    &\stackrel{iii)}{\leq} \sum_{i=1}^{d} \sum_{j=1}^B \left | \left ( \sum_{z=1}^d W_{iz}^{(\ell)}\left (H^{(\ell - 1)} - H^{(\ell - 1)}_{\mathcal{S}^{(\ell - 1)}_t, :} \right )_{zj} \right ) \right |  \\
    &= \sum_{i=1}^d \sum_{j=1}^B \left | W^{(\ell)}_{i, :} \cdot \left (H^{(\ell - 1)} - H^{(\ell - 1)}_{\mathcal{S}^{(\ell - 1)}_t, :} \right )_{:, j} \right | \\
    &\stackrel{iv)}{\leq} \sum_{i=1}^d \sum_{j=1}^B \|W^{(\ell)}_{i, :} \|_1 \left \| \left (H^{(\ell - 1)} - H^{(\ell - 1)}_{\mathcal{S}^{(\ell - 1)}_t, :} \right )_{:, j} \right \|_1\\
    &= \sum_{i=1}^d \sum_{j=1}^B \left ( \sum_{k=1}^d \left |W^{(\ell)}_{ik} \right | \right) \left ( \sum_{z=1}^d \left| \left (H^{(\ell - 1)} - H^{(\ell - 1)}_{\mathcal{S}^{(\ell - 1)}_t, :} \right )_{zj} \right| \right ) \\
    &= \left ( \sum_{i=1}^{d} \sum_{k=1}^d \left| W^{(\ell)}_{ik} \right | \right ) \left (\sum_{z=1}^d \sum_{j=1}^B \left| \left (H^{(\ell - 1)} - H^{(\ell - 1)}_{\mathcal{S}^{(\ell - 1)}_t, :} \right )_{zj} \right| \right ) \\
    &= \left \|\texttt{vec}(W^{(\ell)}) \right \|_1 \cdot \left \|\mu \left ( H^{(\ell - 1)} - H^{(\ell - 1)}_{\mathcal{S}^{(\ell - 1)}_t, :}\right ) \right \|_1 \\
    &\stackrel{v)}{\leq} \sqrt{d} \left \|\texttt{vec}(W^{(\ell)}) \right \|_1 \cdot \left \|\mu \left ( H^{(\ell - 1)} - H^{(\ell - 1)}_{\mathcal{S}^{(\ell - 1)}_t, :}\right ) \right \|_2 \\
    &\stackrel{vi)}{\leq} \sqrt{d} \left \|\texttt{vec}(W^{(\ell)}) \right \|_1 \left ( (0.444)^t \|\mu(H^{(\ell - 1)})\|_2 + (14 + \frac{7}{\sqrt{s}}) \|\mu(\Xi^{(\ell - 1)})\|_1 \right) \\
    &\stackrel{vii)}{\leq} \sqrt{d} \left \|\texttt{vec}(W^{(\ell)}) \right \|_1 \left ( (0.444)^t \|\mu(H^{(\ell - 1)})\|_2 + (14 + \frac{7}{\sqrt{s}}) \left (\| \mu(E^{(\ell - 1)} \|_1 + \|\mu(\Delta^{(\ell - 1)}\|_1 \right) \right) \\
    &\approx \sqrt{d} \left \|\texttt{vec}(W^{(\ell)}) \right \|_1 (14 + \frac{7}{\sqrt{s}}) \left (\| \mu(E^{(\ell - 1)} \|_1 + \|\mu(\Delta^{(\ell - 1)}\|_1 \right)
\end{align*}
where $i)$ holds because $\Delta^{(\ell)}$ simply characterizes the difference between pruned and dense network hidden representations, $ii)$ holds due to properties of the ReLU function, $iii)$ hold by expanding the expression as a sum and replacing the ReLU operation with absolute value, $iv)$ holds due to the Cauchy Schwarz inequality, $v)$ holds due to properties of $\ell_1$ norms, $vi)$ holds due to Lemma \ref{L:hidden_res}, and $vii)$ holds due to the triangle inequality.
As can be seen within the above expression, the value of $\|\mu(\Delta^{(\ell)})\|_1$ can be expressed with respect to $\|\mu(\Delta^{(\ell - 1)})\|_1$.
Now, by beginning with \eqref{eq:multi_last_layer} and unrolling the equation above over all $L$ layers of the network, we obtain the following
\begin{align*}
    \|V^{(L)}_t\|_F \leq \mathcal{O} \left ( \sum_{i=1}^L d^{\frac{L-i}{2}} \left (14 + \frac{7}{\sqrt{s}}\right)^{L - i + 1} \left \|\mu(E^{(i)}) \right \|_1 \left ( \|W^{(L)}\|_F \prod_{j=1}^{L - i} \|\texttt{vec}(W^{(j)})\|_1 \right) \right )
\end{align*}
Now, assume that $H^{(\ell)}$ is $p^{(\ell)}$-row-compressible with factor $R^{(\ell)}$ for $\ell \in [L - 1]$.
Then, defining $p = \max_i p^{(i)}$ and $R = \max_i R^{(i)}$, we invoke Lemma \ref{L:noise_bound} to arrive at the desired result
\begin{align*}
    \|V^{(L)}_t\|_F \leq \mathcal{O} \left ( \sum_{i=1}^L \left (14 + \frac{7}{\sqrt{s}}\right)^{L - i + 1} \left (\|W^{(L)}\|_F   \prod_{j=1}^{L - i} \|\texttt{vec}(W^{(j)})\|_1 \right)  \left (\frac{ d^{\frac{L-i}{2}} s^{1 - \frac{1}{p}}}{ \frac{1}{p} - 1} \right) \right )
\end{align*}
where $R$ is factored out because it has no asymptotic impact on the expression.

\end{proof}

\end{document}